\documentclass[11pt]{article}
\usepackage[margin=1in]{geometry}
\usepackage{amsmath,amssymb,amsthm,mathtools}
\usepackage{booktabs,microtype,enumitem}
\usepackage{tikz,graphicx}
\usepackage{float}
\usepackage[hidelinks]{hyperref}
\usetikzlibrary{arrows.meta,positioning,calc}
\newtheorem{theorem}{Theorem}[section]
\newtheorem{proposition}[theorem]{Proposition}
\newtheorem{lemma}[theorem]{Lemma}
\newtheorem{corollary}[theorem]{Corollary}
\newtheorem{definition}[theorem]{Definition}
\newtheorem{assumption}[theorem]{Assumption}
\newtheorem{remark}[theorem]{Remark}
\theoremstyle{definition}
\newtheorem{example}[theorem]{Example}
\newcommand{\R}{\mathbb R}
\newcommand{\norm}[1]{\left\lVert#1\right\rVert}
\newcommand{\ip}[2]{\left\langle#1,#2\right\rangle}
\newcommand{\smin}{\sigma_{\min}}
\newcommand{\sym}{\operatorname{sym}}

\title{Certified Inference and Training for Deep Equilibrium Networks: A Continuation Framework with Polynomial Complexity Guarantees}
\author{Alex Borisevich, \href{mailto:akpc806b@gmail.com}{akpc806b@gmail.com}}
\date{September 2026}
\begin{document}
\maketitle
\begin{abstract}
We develop a certified continuation framework for inference and training in deep equilibrium networks (DEQs), with training posed as interpolation to accuracy $2^{-b}$. For inference, input homotopy selects an equilibrium branch from a supplied start root, and a rounded tracker follows it under quantitative conditioning, derivative, boundary, and tube-radius certificates. The framework includes structured factorized certificates, sequential block elimination, inheritance of contraction guarantees in adapted coordinates, and bordered continuation through simple folds. For smooth multidimensional DEQs, including tanh networks, rational local tests can construct and validate oriented continuation charts under explicit geometric promises.

For training, programmable dormant bilinear rank-one channels provide output-preserving residual-aligned repairs. Loaded Tikhonov solves diagnose insufficient parameter-to-output directions, while certified gate realization, column stability, well-posed inference, and finite-update error budgets control each pass. Under polynomially bounded certificate, encoding, precision, and backend costs, both inference and training have bit complexity $O(\mathrm{poly}(L+b))$, where $L$ is the encoded instance length; training uses $O(b+\ell)$ passes and reserve channels from an initial residual bounded by $2^\ell$. A budgeted implementation returns either certified success or inconclusive termination. The quantitative core and local certificate machinery are machine-checked in Lean 4, while numerical experiments illustrate the training mechanism.
\end{abstract}

\section{Introduction}
Deep equilibrium networks (DEQs) define a hidden state by an equilibrium
equation instead of evaluating a prescribed sequence of layers. This gives
an implicit architecture whose effective depth is determined by the
equilibrium computation, with training memory advantages that motivated
the original DEQ model~\cite{bai2019}. The architecture also raises basic
theoretical questions: does an equilibrium exist, which equilibrium does
the solver select, and how much computation is needed as the requested
accuracy increases? Research on monotone architectures, Jacobian
regularization, and homotopy methods addresses different parts of these
questions~\cite{winston2020,bai2021,ding2023homoode}.

Training adds a second difficulty. Even when every state can be computed
reliably, a small change in the parameters may fail to move the predictions
in a needed direction. We formulate training as interpolation: for a fixed
batch, produce parameters whose prediction vector is within $2^{-b}$ of
the target. The relevant complexity parameter is the number of requested
bits $b$. A guarantee must therefore control both the number of updates
and the precision and cost of the calculations inside each update.

Continuation organizes the two computational problems. For inference,
we vary the input along a prescribed path and continue an equilibrium
from a supplied start root. For training, we prescribe a straight path
in batch-output space and seek parameter updates that approximately lift
its displacement. The state Jacobian controls equilibrium continuation;
the parameter-to-output Jacobian controls the training lift. The training
proof uses finite approximate lifting steps, without assuming an exact
continuous parameter path.

As a structured realization of the inference contracts, we study small
local recurrent blocks coupled through factors $M,C,N$. This factorization exposes
equilibrium conditioning through an interface matrix, even when the
recurrence is non-normal or the usual fixed-point iteration is not
contractive. Its dimension need not be small relative to the state
dimension for the stated identities and conditioning bounds to hold.
For equilibrium computation, input homotopy follows the selected root
through a certified region. Our
compact continuation theorem and rounded tracker turn bounds on the
inverse Jacobian, derivatives, and tube radius into a finite precision
schedule and polynomial cost
(Theorems~\ref{thm:input-sweep} and~\ref{thm:rounded-tracker}).
Section~\ref{sec:concrete-backend} supplies a rational $\tanh$ evaluator
and an explicit scalar DEQ implementation with a derived polynomial work
envelope, rather than an assumed efficient numerical oracle.

For training, a loaded Tikhonov solve either supports a successful
interpolation pass or identifies a weak output direction. An unused
channel can then be programmed while dormant and activated without
changing the current predictions. The important requirement is that its
new parameter derivative approximate the current pass residual and remain
useful throughout the certified pass region. Detecting a weak direction
alone does not ensure that the architecture can realize this repair.
Our analysis includes programming error, column variation, inexact
evaluation and solves, and rounding. With these contracts, at most
$3(b+\ell)$ passes and channels suffice from an initial residual bounded
by $2^\ell$ (Theorem~\ref{thm:main}).

The resulting claims concern a specified certified promise class.
Polynomial computation requires polynomially bounded geometric and
precision budgets, efficient gate realization, and numerical backends
whose costs include certificate checking. Factorization and smoothness alone
do not establish those requirements. This formulation separates the
proved quantitative mechanism from the architectural certificates still
needed to apply it to a particular DEQ. Lean~4 checks the mathematical
core, while the numerical comparisons provide illustrations of the
loaded trigger rather than a validation of every certificate.
Theorem~\ref{thm:budgeted-certification} makes the operational outcome
explicit: certified success or inconclusive termination within a fixed
polynomial budget. Inconclusive does not mean intrinsically hard.
Valid old inference certificates retain their original solver and require
zero inference repairs. Proposition~\ref{prop:block-coverage} extends
coverage by solving contractive recurrent blocks in dependency order;
large coupling between blocks costs extra precision bits rather than a
polynomial magnitude assumption. This strictly relaxes our uniform
Euclidean inverse-bound promise, while drawing on established block and
weighted-norm contraction ideas rather than superseding them.
Section~\ref{sec:solver-inheritance} makes certified continuation the
default mechanism: an instance certified for an established DEQ solver
is solved polynomially by one derivative-free continuation constructed
from that solver's contractive update, including contraction in solver
coordinates even when the original state update is not contractive. The
outlined exceptions, such as triangular ReLU systems solved by direct
substitution, are direct polynomial solutions outside the continuation
route. Section~\ref{sec:fold-continuation} then extends certified inference
beyond the strict-contraction regimes inherited above. At a simple fold two
equilibrium branches merge and the update derivative has eigenvalue $1$,
so no induced-norm strict-contraction certificate can hold at the fold.
Crossing folds also changes what inference means: with several
coexisting equilibria the requested object is a selected branch,
specified by a certified seed, an orientation, and a terminal crossing,
and without such a selection rule the input--output relation of a
multistable network is not a function. The fold-crossing mechanism
itself is classical, and rigorous per-instance continuation exists in
validated numerics; the contribution here is certified selected-branch
inference for multistable equilibrium networks as a uniform
bit-complexity theorem with a machine-checked local layer. A scalar
tanh example, the canonical bistable neuron, crosses two such folds
with constant geometric bounds. The extension requires a polynomially
bounded branch length and a certified transverse terminal crossing; it
is not an unconditional guarantee for arbitrary equilibrium paths.
Section~\ref{sec:automatic-continuation} constructs rational local
certificates and assembles overlapping oriented charts without a supplied
chart list. Its polynomial guarantee still requires quantitative geometric
bounds and a sound terminal-selection monitor. This is an automatic
local construction within a stated promise class, not a general decision
procedure for equilibrium complexity.
The bird's-eye overview in Section~\ref{sec:bird} explains the training
mechanism before the detailed analysis. Section~\ref{sec:exponential-pass}
shows that a successful repair with well-conditioned inference can still
require exponentially many updates under the global finite-pass
certificate, although a different certified update solves that example
efficiently. The final sections discuss composition, experiments, and
the scope of the guarantees.

\section{Related work and computational complexity}
\label{sec:related}

\subsection{Equilibrium architectures and inference guarantees}
Homotopy-based inference also predates the present construction.
Ding et al.'s HomoODE~\cite{ding2023homoode} connects DEQs and neural ODEs
through homotopy continuation. Beltr\'{a}n and Leykin~\cite{beltran2013}
give certified rational tracking with condition-dependent bit complexity
for polynomial systems. Their polynomial-system theorem cannot be applied
verbatim to a network with transcendental activations. Our input-sweep
result instead states explicit tube, derivative and evaluation promises;
the structured factorization supplies one source of local and interface
inverse bounds, without restricting the general continuation theorem to
that architecture.
This quantitative certificate is the relevant distinction from generic
use of a homotopy solver.

The original DEQ formulation computes a weight-tied network's equilibrium
by root finding and differentiates the selected state implicitly, avoiding
storage proportional to the effective depth~\cite{bai2019}. Multiscale
DEQs extend this construction to coupled representations at several
resolutions~\cite{bai2020}. This memory advantage does not remove the work
of the forward root solve or the backward linear solve. In particular,
factored inverse updates in Broyden's method are already part of
the DEQ literature; they must be distinguished from the architectural
factorized recurrent coupling studied here. Structured connectivity
also predates DEQs in recurrent-network theory, where it is used to relate
connectivity geometry to low-dimensional dynamics and computation
\cite{mastrogiuseppe2018}. The present architecture should therefore be
assessed through its explicit inference and repair certificates, rather
than through the recurrence factorization alone.

Well-posed implicit models also have a substantial history. Implicit deep
learning treats fixed-point prediction rules as a common language for
feedforward and recurrent architectures and studies sufficient conditions
for well-posedness~\cite{elghaoui2021}. Monotone operator equilibrium
networks construct a parameterization with a unique equilibrium and
convergent operator-splitting solvers~\cite{winston2020}. NEMON uses
non-Euclidean contraction theory to obtain well-posedness, fixed-point
algorithms, and quantitative input-output bounds~\cite{jafarpour2021}.
Thus, guaranteed equilibrium evaluation and stability beyond a naive
Euclidean contraction test are established objectives, rather than new
claims of this paper.

\paragraph{Scope of the additional inference result.}
The budgeted formulation in Theorem~\ref{thm:budgeted-certification}
separates sound success, bounded runtime on every input, and coverage of
specific certificate classes. A clock and an inconclusive outcome are
standard algorithmic devices, not a new way to decide intrinsic hardness.
Proposition~\ref{prop:block-coverage} adds explicit binary-precision
accounting for sequential contractive blocks: inter-block couplings may
have exponential magnitude while their accumulated error amplification
requires only polynomially many extra bits. The previous continuation
class is retained as a zero-repair route, including cases that do not
satisfy the new block contraction test.

The exact extension proved is over this paper's uniform Euclidean
inverse-bound certificate, as witnessed by
Proposition~\ref{prop:block-separation}; it is not a strict containment
claim over the classes of El Ghaoui et al. or NEMON
\cite{elghaoui2021,jafarpour2021}. Block structure, changes of norm, and
contractive implicit evaluation already belong to that literature.
Indeed the separation example also admits a simple weighted contraction
certificate. Relative to those antecedents, the result provides an
explicit requested-state-accuracy budget, a compositional success or
inconclusive interface, and machine-checked preservation and precision
lemmas. We do not claim priority for the underlying contraction or
elimination mechanism, or a universal automatic certificate generator.

Our inference results specialize this line of analysis to block-local
recurrence and a normalized feedback interface. The determinant
and inverse reduction use classical matrix identities~\cite{hager1989}.
The non-normal block transfer estimate likewise has a classical resolvent
interpretation. For $A_0=I-W$ and $0<d<1$, write $t=d/(1-d)$; then
\[
 d(I-W)(I-dW)^{-1}
   =tA_0(I+tA_0)^{-1}=I-(I+tA_0)^{-1}.
\]
The symmetric-part assumption makes $A_0$ monotone, and the complementary
resolvent is nonexpansive by standard monotone-operator estimates
\cite{ryu2016}. The paper's role is to combine this estimate with the
normalization $N=(I-W_0)^TM$, exposing the reduced denominator and its
dimension-independent conditioning margin. These are explicit
architectural consequences of established tools, not a new general
resolvent theorem. The block-isotropy restriction and the separate
equilibrium-selection requirement are essential to that interpretation.

\subsection{Reducing the cost of DEQ differentiation}
Jacobian regularization improves forward and backward stability at modest
additional computational cost~\cite{bai2021}. Jacobian-Free
Backpropagation (JFB) avoids the implicit Jacobian solve through an
alternative update~\cite{fung2022}, whereas SHINE reuses quasi-Newton
inverse estimates from the forward computation to approximate the backward
response~\cite{ramzi2022}. More recently, Lipschitz multiscale DEQs study
architectural restrictions that guarantee convergence of forward and
backward fixed-point computations~\cite{sato2026}. A recent preprint on
response renormalization targets selected nearly singular, loss-sensitive
adjoint response channels~\cite{silva2026}.

These methods address closely related numerical bottlenecks, but their
operators must be distinguished. The state residual Jacobian
$J_x=I-DW_{\mathrm{base}}$ governs equilibrium regularity and implicit
sensitivities. The training Jacobian $J=DF(\theta)$ maps parameter changes
to batch prediction changes. Our normal solve regularizes $JJ^*$ and
detects a poorly liftable \emph{training displacement}; a dormant repair
changes the available parameter-to-output directions. This differs from
approximating or selectively damping the state adjoint while leaving the
available training directions unchanged. A bound on $J_x^{-1}$ helps
evaluate $J$ and the propagated gate, but does not by itself give a lower
bound on the training Jacobian or make that gate surjective.

\subsection{Training convergence and the meaning of polynomial complexity}
Global training convergence is already known for restricted implicit
networks. Kawaguchi proves global linear convergence for
implicit layers with nonlinearity on the weight parameterization and
relates their dynamics to a trust-region Newton method~\cite{kawaguchi2021}.
Gao et al. analyze over-parameterized ReLU implicit
networks~\cite{gao2022}. Ling et al. prove linear convergence of gradient
descent under quantitative initialization conditions, establish these
conditions through an over-parameterization analysis, and preserve a unique
equilibrium throughout training~\cite{ling2023}. Truong extends this line
of analysis to activation functions with bounded first and second
derivatives~\cite{truong2025}. Consequently, geometric loss decay or
logarithmic dependence on inverse error cannot be claimed as new solely
because the model is a DEQ.

Three computational statements should be kept separate. First, a linear
algebra operation count measures arithmetic work at a prescribed state and
parameter. Second, an iteration bound to reach tolerance $\epsilon$
depends on convergence constants, including conditioning and stability
margins. Third, polynomial \emph{bit} complexity for
$\epsilon=2^{-b}$ additionally requires polynomially bounded encodings,
working precision, and evaluation costs. For a uniform-inverse tracker
this includes inverse-margin bounds; block elimination can instead charge
the logarithm of inter-block amplification to working precision. For example,
contraction with factor $\kappa$ requires an iteration count proportional
to $\log(1/\epsilon)/[-\log\kappa]$; an exponentially small
$1-\kappa$ can invalidate a polynomial bound even though convergence is
linear. This observation concerns the resources appearing in a rate, not
a contradiction of the cited convergence theorems.

Our result follows the third convention conditionally: the finite-step
error budget and repair reserve are explicit, and the complete bit cost
is reduced to the certified backend contracts in
Assumption~\ref{ass:class}. Since those contracts include an efficient
\emph{whole-pass} backend, Theorem~\ref{thm:main} is a convergence and cost
composition theorem; it does not independently construct a polynomial-time
trainer for arbitrary DEQs. Establishing such a trainer for a concrete
family requires deriving the gate, region, encoding, and backend bounds
from that family's primitive operations.

Classical hardness results for specified neural-network training problems
\cite{blum1992} motivate careful formulation of the admissible class. They
are not, without a reduction, hardness results for the smooth DEQ
architecture or the adaptive reserve considered here. Conversely,
well-conditioned promises and architectural growth change the problem
being solved; they do not refute those hardness results.

\subsection{Function-preserving growth and residual-aligned repair}
Function-preserving changes of network architecture precede the present
construction: Net2Net transfers a trained function to a wider or deeper
network, and network morphism develops a broader framework for such
transformations~\cite{chen2016,wei2016}. Factorized parameter updates are
also established through adapter methods such as LoRA~\cite{hu2022}.
In particular, neither output preservation at activation nor a rank-one
parameter update is sufficient on its own to establish novelty.

Architectural augmentation also has theoretical optimization guarantees.
Liang et al. show, under their classification and loss assumptions, that
adding a special neuron together with a regularizer makes every local
minimum global~\cite{liang2018}. A landscape statement of that kind does
not by itself bound the iterations, precision, or work needed to find a
minimum. Our finite-pass statement instead concerns an executed trajectory
and its quantitative error budgets, conditional on realizable repairs and
certified regions.

Lawton, Galstyan, and Ver Steeg use Gauss--Newton approximations to learn
and evaluate candidate network morphisms~\cite{lawton2024}. This is a
particularly relevant comparison: second-order guidance of architecture
growth is already present in the literature. Our objective is different
from ranking candidate growth operations by an approximate decrease in
loss. We require a bounded propagated column approximating the fixed pass
residual, a derivative-variation bound protecting that column on the pass
region, and a numerical budget implying an actual finite-pass contraction.
The reserve bound then follows from securing one displacement per pass.
It counts activated channels; it does not guarantee realization of
arbitrary batch directions by a rank-one gate or bound storage independently
of the write dimension.

\subsection{Contribution and remaining boundary}
For inference, Section~\ref{sec:solver-inheritance} makes a single
continuation mechanism the default route for the quantitative
contraction regimes certified in the established solver coordinates,
with directly certified exceptions outlined explicitly. This is
preservation of their guarantees, not a claim that those methods require
the Euclidean inverse bound used by our initial sweep theorem. Section~\ref{sec:fold-continuation} separately uses
classical pseudo-arclength regularity~\cite{dickson2006} to pass simple
folds. Fold crossing itself is not new: validated continuation certifies
branches through folds per instance~\cite{berg2021validated,dickson2006},
and homotopy has been used in DEQ inference for
acceleration~\cite{ding2023homoode}. What appears to be new for
equilibrium networks is the problem statement and its complexity:
certified selected-branch inference for multistable models beyond
strict-contraction regimes. At a fold the update derivative has eigenvalue
$1$, so no induced-norm strict-contraction certificate can hold there. The
result is stated as a uniform bit-complexity theorem over encoded
instances with a machine-checked local layer, and composed with rounded
tracking, selected-terminal certification, and budgeted termination.
Subhomogeneous models~\cite{sittoni2024} are included through their
metric contraction only when the coordinate and precision costs are
also controlled.

The contribution studied here is the quantitative connection between
loaded Tikhonov diagnosis, output-preserving residual-aligned repair,
protection under column drift, and finite inexact passes. The mechanism
permits a precision-scaled repair reserve without requiring a uniform
right inverse of the training Jacobian in every output direction. The
factorized inference certificate supplies complementary control of state
sensitivities. The machine-checked layer covers the stated abstract
implications and their hypotheses; it supports mathematical reliability
rather than replacing architectural realizability or experimental
validation.

The comparison motivates a certified adaptive-training framework, rather
than a claim of the first convergent DEQ trainer, the first use of
Gauss--Newton for growth, or unconditional polynomial-time neural-network
training. The outstanding step toward a broader algorithmic result is an
efficient certificate construction for a nontrivial DEQ family. The
historical numerical runs reported later illustrate the intended mechanism
but do not certify that step or test all contracts of the finite-pass
algorithm.

\section{The DEQ architecture and computational problem}
For input $u\in\R^{d_u}$, consider
\begin{equation}
 x=\varphi\bigl(W_0x+MCN^Tx+Uu+c\bigr),\qquad x\in\R^n.
 \label{eq:base}
\end{equation}
Here $W_0$ is block diagonal with bounded block sizes, $M,N\in\R^{n\times r}$,
and $C\in\R^{r\times r}$. The readout is $\widehat y=w^Tx+c_o$.
Write $W_{\mathrm{base}}=W_0+MCN^T$ for the fixed recurrent weight used
during inference.
The coupling dimension $r$ is an explicit computational resource, not a
smallness assumption: the factorization identities below hold for any
$r$, and $MCN^T$ can have full rank when $r=n$. A smaller interface
dimension can reduce linear-algebra work. This structured model is one
way to instantiate the general inference contracts; arbitrary recurrence
still requires sound conditioning, domain, and evaluation certificates.
The activation acts componentwise and is $C^{1,1}$ on the certified
preactivation region. For the monotone activation interpretation we require
$0\le\varphi'\le1$, rather than only $|\varphi'|\le1$; $\tanh$ is the running
example. The abstract training results apply to any $C^{1,1}$ output map
with the required evaluation contracts, independently of this derivative
range.

Figure~\ref{fig:architecture} shows the local and factorized recurrent contributions
and their shared equilibrium feedback.
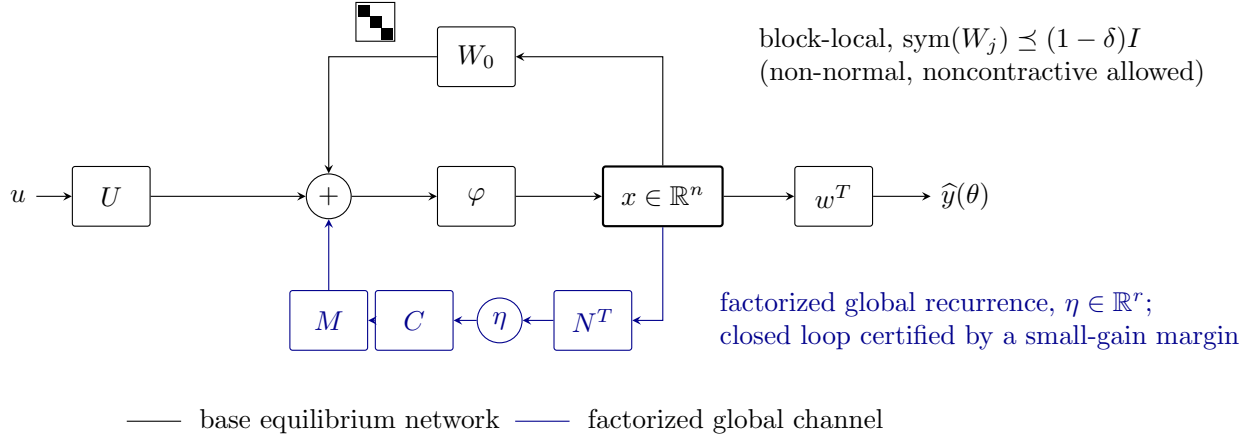
\begin{figure}[H]
\centering
\resizebox{\linewidth}{!}{%
\begin{tikzpicture}[>=stealth,
  box/.style={draw,rounded corners=1pt,minimum height=2em,minimum width=2.6em,
              fill=white,font=\small},
  op/.style={draw,circle,inner sep=1pt,minimum size=1.5em,fill=white,
             font=\small},
  every node/.style={font=\small}]

\node[op]  (sum) at (0,0)   {$+$};
\node[box] (act) at (1.9,0) {$\varphi$};
\node[box,minimum width=4.0em,thick] (x) at (4.3,0) {$x\in\R^{n}$};
\node (u) at (-4.0,0) {$u$};
\node[box] (U) at (-2.8,0) {$U$};
\draw[->] (u)   -- (U);
\draw[->] (U)   -- (sum);
\draw[->] (sum) -- (act);
\draw[->] (act) -- (x);
\node[box] (w)    at (6.5,0) {$w^{T}$};
\node      (yhat) at (8.2,0) {$\widehat y(\theta)$};
\draw[->] (x) -- (w);
\draw[->] (w) -- (yhat);

\node[box] (W0) at (1.9,1.8) {$W_0$};
\draw[->] (x.north) |- (W0);
\draw[->] (W0) -| (sum.north);
\begin{scope}[shift={(0.35,2.0)},scale=0.5]
 \draw (0,0) rectangle (1,1);
 \fill (0.06,0.64) rectangle (0.36,0.94);
 \fill (0.36,0.34) rectangle (0.66,0.64);
 \fill (0.66,0.04) rectangle (0.96,0.34);
\end{scope}
\node[align=left,anchor=west] at (5.4,1.8)
 {block-local, $\sym(W_j)\preceq(1-\delta)I$\\
  (non-normal, noncontractive allowed)};

\begin{scope}[blue!55!black]
\node[box,draw=blue!55!black] (Nt)  at (3.4,-1.6) {$N^{T}$};
\node[op,draw=blue!55!black]  (eta) at (2.2,-1.6) {$\eta$};
\node[box,draw=blue!55!black] (Cb)  at (1.1,-1.6) {$C$};
\node[box,draw=blue!55!black] (M)   at (0,-1.6)   {$M$};
\draw[->] (x.south) |- (Nt);
\draw[->] (Nt)  -- (eta);
\draw[->] (eta) -- (Cb);
\draw[->] (Cb)  -- (M);
\draw[->] (M)   -- (sum.south);
\node[align=left,anchor=west] at (4.9,-1.6)
 {factorized global recurrence, $\eta\in\R^{r}$;\\
  closed loop certified by a small-gain margin};
\end{scope}

\begin{scope}[shift={(-2.6,-2.9)}]
\draw[-] (0,0) -- (0.7,0);
\node[anchor=west] at (0.8,0) {base equilibrium network};
\draw[-,blue!55!black] (5.0,0) -- (5.7,0);
\node[anchor=west] at (5.8,0) {factorized global channel};
\end{scope}
\end{tikzpicture}}
\caption{An equilibrium network with factorized recurrence. The
equilibrium state solves $x=\varphi(W_0x+MCN^{T}x+Uu+c)$.  Black:
block-local recurrence satisfying the stated symmetric-part bound.
Blue: global recurrence through the interface state $\eta=N^{T}x$.
Under the normalized-family hypotheses, its interface admits a
small-gain certificate. The identities hold for any interface dimension
$r$; a smaller dimension can reduce the cost of the $r\times r$ solve.}
\label{fig:architecture}
\end{figure}

For a fixed batch, stack the scalar predictions as $F(\theta)\in\R^m$ and
write $J(\theta)=DF(\theta)$. The trainable vector includes the active
parameters and any programmed vectors that can subsequently vary.

Given rational instance data and a precision request $b\in\mathbb N$, the
computational objective is a rational parameter encoding $\widehat\theta$
with
\begin{equation}
 \norm{F(\widehat\theta)-y^*}\le2^{-b}.
 \label{eq:objective}
\end{equation}
Precision complexity is measured in $b$, rather than in $2^b$. We call this
arbitrarily accurate interpolation; no finite exact rational solution is
asserted for a general nonlinear network. Parameter dimension is denoted
$p$ and the initial-residual exponent by $\ell$.

\section{Certified inference}
The continuation and tracking results concern a general residual map
$H(s,x)$ with the stated certificates. We first derive an architectural
source of inverse bounds for the factorized model, then state the
well-posedness and tracking contracts it can support.
\subsection{Exact factorized Jacobian reduction}
For the base model~\eqref{eq:base}, let $D$ be the activation derivative matrix. Set
\begin{equation}
 A=I-DW_0,\qquad B=DMC,\qquad L=N^T,
 \qquad Q=I-LA^{-1}B.
 \label{eq:reduced}
\end{equation}
The state Jacobian is $A-BL$. This definition fixes the ordering of the
reduced factors; $C$ need not commute with the reduced transfer.

\begin{theorem}[Inference reduction]\label{thm:reduction}
If $A$ is invertible, then
\[
 \det(A-BL)=\det(A)\det(Q).
\]
The full and reduced kernels are explicitly isomorphic. If $Q$ is
invertible, then
\[
 (A-BL)^{-1}=A^{-1}+A^{-1}BQ^{-1}LA^{-1}.
\]
In particular, bounds $\norm{A^{-1}}\le a$, $\norm B\le\beta$,
$\norm L\le\lambda$, $\norm{Q^{-1}}\le q$ give
\[
 \norm{(A-BL)^{-1}}\le a+a^2\beta q\lambda.
\]
The same inverse-norm bound controls the adjoint sensitivity solve.
\end{theorem}
\begin{proof}
The activation chain rule applied to the state residual
$x-\varphi(W_{\mathrm{base}}x+Uu+c)$ gives
$J_x=I-DW_{\mathrm{base}}$. Splitting local recurrence and global coupling
gives $J_x=A-BL$. Factorization gives
\[
 A-BL=A(I-A^{-1}BL).
\]
The rectangular determinant identity
$\det(I-XY)=\det(I-YX)$ with $X=A^{-1}B$, $Y=L$ proves the determinant
formula. No commutation of the latent gain with the reduced transfer is
used.

If $(A-BL)x=0$, multiplication by $A^{-1}$ gives
$x=A^{-1}BLx$. Therefore $u=Lx$ satisfies $Qu=0$. Conversely, if
$Qu=0$, then $x=A^{-1}Bu$ satisfies $(A-BL)x=0$ and $Lx=u$.
These maps are inverse linear maps between the kernels; in particular
nonzero kernel vectors correspond.

For invertible $Q$, let
$S=A^{-1}+A^{-1}BQ^{-1}LA^{-1}$. Use
\[
 (I-LA^{-1}B)Q^{-1}=I,
 \qquad Q^{-1}(I-LA^{-1}B)=I
\]
to check $(A-BL)S=S(A-BL)=I$ by expanding the products. Submultiplicativity
and the triangle inequality then give
\[
 \norm S\le\norm{A^{-1}}+
       \norm{A^{-1}}^2\norm B\norm{Q^{-1}}\norm L
       \le a+a^2\beta q\lambda.
\]
The adjoint has the same induced norm, so
$\norm{J_x^{-*}}=\norm{J_x^{-1}}$.
\end{proof}

\subsection{Non-normal local blocks}
For the normalized family, set $A_0=I-W_0$, $N=A_0^TM$,
$\norm M\le1$, and $C=\alpha I$, with $0\le\alpha<1$. Assume
\begin{equation}
 D=\operatorname{blkdiag}(d_jI_{k_j}),\quad 0\le d_j\le1,
 \qquad \sym(W_j)\preceq(1-\delta)I,
 \quad 0<\delta\le1.
 \label{eq:block}
\end{equation}
Block isotropy is a substantive condition. Ordinary elementwise $\tanh$
does not satisfy it automatically on multi-coordinate blocks.

\begin{theorem}[Local transfer and reduced margin]\label{thm:transfer}
Under \eqref{eq:block}, each local Jacobian satisfies
$\smin(I-d_jW_j)\ge\delta$ and
\[
 \norm{d_j(I-W_j)(I-d_jW_j)^{-1}}\le1.
\]
Consequently $T=A_0(I-DW_0)^{-1}D$ has norm at most one and
\[
 Q(D)=I-\alpha M^TTM,\qquad \smin(Q(D))\ge1-\alpha.
\]
No bound $\norm{W_j}<1$ is required.
\end{theorem}
\begin{proof}
For one local block, the strict symmetric-part bound yields
\begin{align*}
 \ip{(I-dW)x}{x}
 &\ge[1-d(1-\delta)]\norm x^2\\
 &\ge\delta\norm x^2,
\end{align*}
where the last inequality uses $0\le d\le1$ and $0<\delta\le1$.
Cauchy--Schwarz gives
$\delta\norm x\le\norm{(I-dW)x}$ for $x\ne0$, and the inequality is
immediate at zero. In finite square dimension this implies invertibility
and inverse norm at most $\delta^{-1}$.

The transfer estimate follows from the exact norm identity
\begin{align*}
 &\norm{x-dWx}^2-\norm{d(x-Wx)}^2\\
 &\quad=(1-d)^2\norm x^2+
    2d(1-d)\bigl(\norm x^2-\ip x{Wx}\bigr)\ge0.
\end{align*}
Only $\sym(W)\preceq I$ is needed for this estimate. Setting
$x=(I-dW)^{-1}y$ proves
\[
 \norm{d(I-W)(I-dW)^{-1}y}\le\norm y.
\]
The Euclidean norm of a block-diagonal operator is the maximum of its
block norms. Since $D$ is scalar on each block, the local operators
commute in the required products and
$T=A_0(I-DW_0)^{-1}D$ is precisely their block sum. Hence $\norm T\le1$.

With $N=A_0^TM$,
\[
 Q=I-\alpha M^TA_0(I-DW_0)^{-1}DM=I-\alpha M^TTM.
\]
The compressed transfer $K=M^TTM$ satisfies $\norm K\le1$. For every $u$,
\[
 \norm{(I-\alpha K)u}\ge\norm u-\alpha\norm{Ku}
       \ge(1-\alpha)\norm u.
\]
This proves the reduced gap and its inverse bound.
\end{proof}

\begin{corollary}[Dimension-independent conditioning]\label{cor:fullgap}
If also $\norm{W_0}\le W$, then the full state Jacobian satisfies
\[
 \smin(J_x(D))\ge
 \frac{\delta^2(1-\alpha)}{\delta+1+W}.
\]
At an unsaturated state $D=I$,
$J_x=(I-\alpha MM^T)A_0$. If the normalized feedback has a unit eigenvector
with eigenvalue one, the reduced gap equals $1-\alpha$ and there is a unit
full-state direction with gain at most $(1+W)(1-\alpha)$. Thus, if such a
state belongs to the inference region, its worst-case gap is
$\Theta(1-\alpha)$ with constants independent of $n$ and $r$.
\end{corollary}
\begin{proof}
The local inverse norm is at most $\delta^{-1}$, $\norm D\le1$,
$\norm M\le1$, and
$\norm N=\norm{A_0^TM}\le1+W$. Theorem~\ref{thm:reduction} gives
\[
 \norm{J_x^{-1}}\le\delta^{-1}+
        \frac{\alpha(1+W)}{\delta^2(1-\alpha)}
 \le\frac{\delta+1+W}{\delta^2(1-\alpha)}.
\]
Taking reciprocals gives the asserted lower bound.

At $D=I$, the normalization gives
$J_x=(I-\alpha MM^T)A_0$ and $Q=I-\alpha M^TM$.
If a unit $u$ satisfies $M^TMu=u$, then
$\norm{Qu}=1-\alpha$, which meets the reduced lower bound.
Also $y=Mu$ is unit and $MM^Ty=y$. Put
\[
 x=\frac{A_0^{-1}y}{\norm{A_0^{-1}y}}.
\]
Then $\norm x=1$ and
\[
 \norm{J_xx}=(1-\alpha)\norm{A_0x}
     \le(1-\alpha)\norm{A_0}\le(1+W)(1-\alpha).
\]
If the unsaturated state is present, this upper bound combines with the
uniform lower bound to give the worst-case scaling. A unit top feedback
eigenvector is a normalization certificate; its existence is not asserted
for zero-dimensional feedback or an unnormalized matrix.
\end{proof}

\begin{proposition}[Anisotropic perturbations]\label{prop:perturbation}
If a certified state Jacobian $J_0$ satisfies
$\norm{J_0v}\ge g\norm v$ and the true Jacobian satisfies
$\norm{J-J_0}\le\varepsilon<g$, then
$\norm{Jv}\ge(g-\varepsilon)\norm v$. In finite square dimension it is
invertible, with inverse norm at most $(g-\varepsilon)^{-1}$.
\end{proposition}
\begin{proof}
For every $v$,
\[
 \norm{Jv}\ge\norm{J_0v}-\norm{(J-J_0)v}
       \ge(g-\varepsilon)\norm v.
\]
The positive lower bound implies injectivity. A linear map on a finite
square space is then surjective. Apply the same lower bound to $J^{-1}z$
to obtain the inverse norm estimate. In particular, estimating only the
change in the activation derivatives is insufficient unless the induced
full Jacobian error is also bounded.
\end{proof}
This permits elementwise activations near a block-isotropic reference when
their induced Jacobian perturbation is quantitatively bounded.

\subsection{Well-posedness and numerical evaluation}
A nonsingular Jacobian is a local regularity certificate. It does not by
itself certify a globally unique equilibrium. The training region must
also carry a sound well-posedness and evaluation certificate.

\begin{proposition}[A complete contraction certificate]\label{prop:wellposed}
If the full layer map $f(x)=\varphi(W_{\mathrm{base}}x+Uu+c)$ is
$\kappa$-contractive, $0\le\kappa<1$, it has exactly one equilibrium $x^*$
and every numerical state $\widehat x$ satisfies
\[
 \norm{\widehat x-x^*}\le
       \frac{\norm{\widehat x-f(\widehat x)}}{1-\kappa}.
\]
A sufficient condition is a nonexpansive activation and
$\norm{W_{\mathrm{base}}}\le\kappa$.
\end{proposition}
\begin{proof}
The Banach fixed-point theorem applies to the complete Euclidean state
space. For a nonexpansive activation,
\[
 \norm{f(x)-f(y)}\le\norm{W_{\mathrm{base}}(x-y)}
       \le\kappa\norm{x-y}.
\]
It gives existence and uniqueness of $x^*$. The triangle inequality and
the contraction estimate give
\[
 \norm{\widehat x-x^*}\le
   \norm{\widehat x-f(\widehat x)}+
   \kappa\norm{\widehat x-x^*}.
\]
Rearranging proves the a posteriori bound. It converts an enclosed layer
residual into an equilibrium-state error; the residual of a floating-point
layer evaluation must itself include the layer evaluation error.
\end{proof}
The contraction subclass supplies one complete evaluation certificate.
Noncontractive instances require another sound state-selection and
evaluation backend, such as a quantitatively justified monotone solver;
small solve residuals alone are insufficient.

With bounded local block size $k$, local linear algebra and the reduced
solve cost $O(nk^2+nr+r^3)$ operations after the required local inverses and
reduced factors are available. Forming all reduced factors can additionally
cost $O(nr^2)$, and applying a precomputed reduced inverse costs $O(r^2)$.
Inference iteration counts, factor bit sizes, and requested accuracy remain
separate complexity resources.
For $r=n$ and bounded $k$, this accounting gives $O(n^3)$ arithmetic
work for the factor formation and solve. The dimensional reduction then
offers no size advantage, but the operation count is still polynomial;
bit complexity additionally requires the stated numerical and encoding
budgets.

\subsection{Certified input homotopy beyond contraction}
\label{sec:inference-homotopy}
An alternative to contraction is to follow a distinguished equilibrium as the input varies along
a supplied sweep $u(s)$, $0\le s\le1$. Write
\[
 H(s,x)=x-\varphi(W_{\mathrm{base}}x+Uu(s)+c).
\]
All model parameters are fixed during the inference sweep. Its state Jacobian
is the same factorized operator
used above. A certified positive reduced margin controls the linear systems
needed by the tangent predictor and Newton corrector, even when iteration
of the layer map is not contractive. The relevant promise concerns the
whole certified tube, rather than just the final equilibrium.

\begin{theorem}[Input-sweep inference certificate]\label{thm:input-sweep}
Let $H$ be $C^2$ on a neighborhood of $[0,1]\times\overline\Omega$, where
$\Omega\subset\R^n$ is bounded and open. Supply a start root $x_0\in\Omega$,
$H(0,x_0)=0$, and exclude zeros on $[0,1]\times\partial\Omega$.
At every zero in this region assume $H_x$ is invertible and
\[
 \norm{H_x^{-1}}\le K,\qquad \norm{H_s}\le B.
\]
Then the start root has a unique continued branch $x(s)$ for the whole
sweep, with
\[
 x'(s)=-H_x^{-1}H_s,\qquad \norm{x'(s)}\le KB.
\]
Its graph length is at most $\sqrt{1+(KB)^2}$. In particular the traversal
quantity $\int_0^1\norm{H_x^{-1}}\sqrt{1+\norm{x'}^2}\,ds$ is at most
$K\sqrt{1+(KB)^2}$.

For numerical tracking, additionally supply a tube of radius $\rho>0$
around the branch on which the inverse bound holds and the full second
derivative is bounded by $H_2$. If $K,B,H_2,\rho^{-1}$, the encoded
start data, and evaluation and solve costs at requested precision are
polynomially bounded in the input size, a tangent predictor with certified
Newton correction tracks this branch to state accuracy $2^{-b}$ with
polynomial total bit cost, provided the evaluator and solver include
rounding and certification costs in those bounds.
\end{theorem}
\begin{proof}
The implicit-function theorem gives a local branch and the differentiated
identity $H_s+H_xx'=0$. The inverse bound gives $\norm{x'}\le KB$.
If a maximal branch stopped at $s_*<1$, this velocity bound would make
$x(s)$ Cauchy as $s\uparrow s_*$. Its limit lies in $\overline\Omega$
and is a zero by continuity. Boundary exclusion puts it in $\Omega$,
where the invertible derivative extends the branch, a contradiction.
Local uniqueness also makes any two continuations from $x_0$ agree.
Integrating the graph speed proves the length estimates.

For the numerical assertion, differentiating the tangent identity bounds
$\norm{x''}$ by $KH_2(1+KB)^2$ using Euclidean product norms. Choose an
inverse-polynomial sweep step so that the predictor error is smaller than
both a fixed fraction of $\rho$ and a fixed fraction of $(KH_2)^{-1}$;
if $H_2=0$, only the tube constraint is needed. Taylor's formula bounds
this error by a constant times $KH_2(1+KB)^2h^2$. Newton's local error
estimate is
$\norm{e_{j+1}}\le KH_2\norm{e_j}^2/2$ while its segment remains in the
tube. Smaller certified solve and evaluation errors preserve a uniform
contraction of the corrector error. An inverse-polynomial step gives
polynomially many stages; $O(b)$ certified correction iterations suffice
at the terminal stage, with polynomially many precision bits. Summing
the supplied bit costs proves the claim. This constructs a tracking
schedule under quantitative tube promises; it does not discover them.
\end{proof}

\paragraph{A finite certified tracker.}
A conservative tracker can use the previous corrected state as its
predictor; a tangent predictor is optional. This simpler choice already
has polynomial complexity under the tube promises and makes branch
selection explicit. Let $V=KB$ and choose $r>0$ with
$r\le\rho$ and $KH_2r\le1/4$. Use $N\ge1$ sweep stages with
$V/N\le r/2$, and write $q_i=x(i/N)$. At each stage perform $m$
rounded Newton corrections for $H((i+1)/N,\cdot)$, starting from the
preceding corrected state. The branch centers $q_i$ are used in the
proof; the algorithm evaluates the equation and its Jacobian, not these
unknown exact centers.

\begin{theorem}[Rounded sweep tracker]\label{thm:rounded-tracker}
On each radius-$r$ branch ball, suppose the state derivative is
$H_2$-Lipschitz and has inverse norm at most $K$. A Newton correction
at $y$ uses a step $v$ and a whole-update rounding vector $z$ satisfying
\[
 \norm{H_xv+H}\le\eta,\qquad \norm z\le\zeta,
 \qquad K\eta+\zeta\le\delta\le r/8.
\]
Here the solve residual includes function and derivative evaluation errors
relative to the true equation. Starting at a state within $r/2$ of $q_0$,
with $m\ge1$ corrections per stage, every corrected stage endpoint is
within $r/2$ of its selected root, and every intermediate correction
remains in that root's radius-$r$ ball. For $N\ge1$ the final state obeys
\[
 \norm{\widehat x_N-q_N}\le2^{-m}r+2\delta.
\]
If $r\le2^\ell$, choose $m=b+\ell+2$ and
$\delta\le\min\{r/8,2^{-(b+2)}\}$ to obtain state error at most $2^{-b}$.
There are exactly $Nm$ corrections. If an encoded correction backend
simulates these updates, has uniformly polynomial bit cost and output
size, and the query size and precision are bounded by $S$ and $P$, its
correction work is at most
$Nm\,A(S+P+1)^d$ for its fixed family constants $A,d$.
Certification of the tube and stage data has its separately supplied cost.
\end{theorem}
\begin{proof}
For a root $q$ and $e=\norm{y-q}$, Taylor's remainder on the convex ball
is at most $H_2e^2$ (a conservative constant). Applying the true inverse
to that remainder and the solve residual gives
\[
 \norm{y+v+z-q}\le KH_2e^2+K\eta+\zeta
                    \le e/4+\delta\quad(e\le r).
\]
Thus correction stays in the ball; after its first iteration the error
is at most $3r/8<r/2$. Induction also gives
$e_j\le2^{-j}r+2\delta$. The branch velocity bound and the mean value
theorem give $\norm{q_{i+1}-q_i}\le V/N\le r/2$. A corrected endpoint
therefore starts the next stage within radius $r$ of its root. Induction
over stages proves safety and the final estimate. The stated precision
schedule bounds the final error by
$3\cdot2^{-(b+2)}\le2^{-b}$.

The encoded and decoded finite recursions agree by induction under the
backend simulation contract. Summing the $Nm$ certified correction
costs gives the work bound. This identifies the computed state whose
error is bounded; a small terminal residual alone would not do so.
\end{proof}

Compactness and local regularity have distinct roles. On the compact zero
set in the certified state region, the implicit-function theorem supplies
local graph charts for projection to the sweep interval. These charts
make the projection a local homeomorphism; compactness makes it a covering
map. Lifting the identity sweep from the start root gives its unique
continuous branch. Local agreement with the implicit function gives the
derivative formula, including one-sided endpoint neighborhoods. This
also explains why a supplied global path is unnecessary.

For the normalized block-isotropic family, Corollary~\ref{cor:fullgap}
supplies $K=(\delta+1+W)/(\delta^2(1-\alpha))$. A nonexpansive activation
gives $B\le\norm U\sup_s\norm{u'(s)}$. Bounded activations also make
boundary exclusion inexpensive: for coordinatewise $\tanh$, every root
has $\norm{x}_\infty<1$, so the box $(-2,2)^n$ excludes boundary roots.
The reduced-margin and derivative certificates must apply to the actual
residual throughout the required region. For general $\tanh$ states,
block-isotropy is an additional
restriction; the anisotropic perturbation certificate can replace it when
its slack is positive.

The conclusion selects the branch reached from the supplied start root;
it does not exclude other disconnected equilibria. It also gives a
controlled warm start for nearby inputs: a certified tube can be reused
while its margin and domain tests remain valid. Unlike a generic
pseudo-arclength argument, this theorem has a nonsingular $H_x$ throughout,
so the input coordinate itself is a valid sweep and no fold bypass is
needed. Section~\ref{sec:fold-continuation} supplies a separate extension
using a bordered margin and a terminal-face certificate in place of
these hypotheses, while retaining this original route without repairs.

\subsection{Certified success or inconclusive termination}
\label{sec:budgeted-certification}
The numerical theorems above are success guarantees on certified classes.
They do not decide whether an arbitrary inference problem has intrinsic
polynomial complexity. We now place these solvers inside a bounded
certification procedure. An inference repair changes the solver,
coordinates, or order of elimination while preserving the original
equilibrium equation and its selected solution. It does not activate a
new recurrent channel and thereby change the network being evaluated.

Fix an encoded backend family, input length $S$, requested accuracy
$2^{-b}$, and a polynomial transition budget $P=P(S,b)$. A transition
includes its certificate tests and returns either a new state with a
repair flag, a certified answer, or \emph{inconclusive}. The wrapper
executes at most $P$ transitions; exhaustion also returns inconclusive.
The successful-answer predicate for inference is
\[
 \mathsf{Good}_{\mathrm{inf}}(y)
 \iff \exists q:\ q\text{ is the specified equilibrium and }
          \norm{\operatorname{decode}(y)-q}\le2^{-b}.
\]
For a multivalued equilibrium problem, ``specified'' includes the
continuation branch from the supplied start root. A small residual alone
does not establish this predicate. Certificate production may be internal
or supplied with the input; its validation and all failed searches are
charged to the computation.

\begin{theorem}[Budgeted certification with preserved coverage]
\label{thm:budgeted-certification}
Suppose the initial state satisfies an invariant, every continuing
transition preserves it, and every accepting transition from an invariant
state establishes $\mathsf{Good}_{\mathrm{inf}}$. Assume that, on
\emph{all} inputs and all transitions reached before the cutoff, each
transition has bit cost at most $C(S+b+P+1)^d$, including certificate
search, verification, arithmetic, and rounding. Allow the same bound for
initialization and reporting. Then:
\begin{enumerate}[label=(\roman*),leftmargin=2em]
\item The wrapper returns a certified answer or inconclusive, with total
bit cost at most $(P+1)C(S+b+P+1)^d$. Every returned answer satisfies the
specified accuracy and equilibrium-selection condition.
\item If the actual strategy has a successful trace of $T\le P$
transitions with $R$ repair flags, the wrapper returns its same answer
after $T$ transitions and reports exactly $R$ repairs.
\item In particular, dispatching a valid existing certificate directly
to its original tracker, without inference repairs, preserves its answer
and has $R=0$, provided $P$ includes its complete certified work.
\end{enumerate}
Inconclusive termination makes no assertion about existence, uniqueness,
or the intrinsic complexity of the requested equilibrium.
\end{theorem}
\begin{proof}
Define the runner recursively in the remaining budget. At zero it returns
inconclusive. Otherwise it executes one transition; a continuing
transition recurses with one less unit and adds one to the transition
count and its flag to the repair count. Induction bounds transitions by
$P$ and repairs by transitions. Induction using invariant preservation
and sound acceptance proves answer soundness. A second induction on a
successful trace proves that any sufficient budget reproduces the trace,
including its answer and repair count. Summing the per-transition bound
and one initialization/reporting allowance gives the displayed cost.
The zero-repair statement is the trace result with $R=0$.
\end{proof}

Coverage concerns the executed strategy: the mere existence of a
certificate does not ensure that an arbitrary search will find it.
We give supplied valid old certificates priority and reserve their full
work budget before attempting any optional repair. For a fixed old
backend family and a fixed extension family with polynomial coverage
bounds, one polynomial $P$ can dominate both bounds and the dispatch
overhead. There is no single prescribed polynomial claimed to dominate
all possible polynomial-time backend families. The wrapper is standard
resource accounting; the additional mathematical content must come from
explicit, efficiently checkable coverage classes.

\subsection{Block elimination with polynomial precision overhead}
\label{sec:block-extension}
We now give an efficiently checkable class whose equilibrium admits
polynomial bit complexity even when its global Euclidean inverse bound is
exponentially large. After a permutation, let a $\tanh$ DEQ have blocks
\begin{equation}\label{eq:block-extension}
 x_i=\tanh\!\left(W_{ii}x_i+\sum_{j<i}W_{ij}x_j+d_i\right),
 \qquad 1\le i\le m,
\end{equation}
with rational encoded weights and inputs. Use the induced infinity norm
and require $\norm{W_{ii}}_\infty\le1-\gamma_i$, where
$0<\gamma_i\le1$ and $\gamma_i^{-1}$ is polynomially bounded in $S$.
Only the \emph{encoding length}, not a polynomial magnitude bound, is
required for the off-diagonal couplings. A proposed partition and the
norm inequalities are checked with rational arithmetic; alternatively,
the directed dependency graph supplies its strongly connected components
and a topological ordering in polynomial time. A failed norm test is
inconclusive for this route.

\begin{proposition}[Constructive block coverage]\label{prop:block-coverage}
For this class, sequential certified block solves compute the unique
equilibrium to infinity-norm accuracy $2^{-b}$ with polynomial total bit
cost in $S+b$. The block decomposition preserves exactly the solutions
of the original equations. It supplies an additional successful route in
Theorem~\ref{thm:budgeted-certification}, with at most one decomposition
repair followed by ordinary numerical solves. The original certificate
route retains priority and uses zero such repairs.
\end{proposition}
\begin{proof}
Each block map is a contraction for every upstream input, so sequential
application of the contraction theorem proves existence and uniqueness.
For two upstream inputs, subtract the fixed-point equations, use that
$\tanh$ is nonexpansive, and move the self-dependence to the left. This
gives the sensitivity estimate
\[
 \norm{\Delta x_i}_\infty\le
 \gamma_i^{-1}\sum_{j<i}\norm{W_{ij}}_\infty
                         \norm{\Delta x_j}_\infty.
\]
Let $A=\max_i\gamma_i^{-1}\sum_{j<i}\norm{W_{ij}}_\infty$.
Choose an integer $a\ge0$ with $1+A\le2^a$, and solve each block to
local error $\delta=2^{-(b+am)}$ for its computed upstream input. The
prefix maximum errors $E_i$ satisfy $E_0=0$ and
\[
 E_i\le\max\{E_{i-1},\delta+A E_{i-1}\}.
\]
Induction gives $E_i\le\delta(1+A)^i$: the previous envelope is
nondecreasing and the new block error obeys
$\delta+A\delta(1+A)^{i-1}\le\delta(1+A)^i$.
Consequently,
\begin{equation}\label{eq:block-precision-budget}
 E_m\le2^{-(b+am)}(1+A)^m\le2^{-b}.
\end{equation}
Since rational row sums and $\gamma_i^{-1}$ have polynomial encoding
length, $a$ and $p=b+am$ are polynomially bounded even when $A$ is
exponential in magnitude.

Start each block iteration at zero. Its exact initial error is at most
one. Set the certified whole-update error to
$\nu_i\le\gamma_i\delta/2$, including affine evaluation, $\tanh$
evaluation, and rounding; coordinatewise clipping to $[-1,1]$ cannot
increase the error to the equilibrium. The inexact contraction recurrence
gives after $k$ iterations
\[
 e_k\le(1-\gamma_i)^k+\nu_i/\gamma_i.
\]
Thus $k\ge\gamma_i^{-1}\log(2/\delta)$ suffices. This is polynomial.
Rational affine operations at these precisions have polynomial cost:
large coefficients need their binary length and the additional guard
bits, not a number of operations proportional to their magnitude.
For a large preactivation $z$, the bound
$1-|\tanh z|\le2e^{-2|z|}$ permits certified saturation to the required
precision; otherwise $|z|$ is polynomial in the working precision and
standard rational exponential approximation suffices. This gives a
polynomial implementation of each certified update. Summing over the
blocks proves the bit bound. For Euclidean state accuracy, replace $b$
by $b+\lceil\log_2 n\rceil$, a conservative norm-conversion allowance.
\end{proof}

The general multidimensional $\tanh$ block implementation and its
primitive bit-cost analysis are separate obligations; they are not
identified with the encoded scalar backend.

\begin{proposition}[Separation from the uniform inverse promise]
\label{prop:block-separation}
The block class contains a family with $O(L)$-bit descriptions whose
Euclidean inverse-Jacobian bound along an input sweep is at least
$2^{L+1}$, but whose equilibrium is computable in polynomial bit cost
in $L+b$.
\end{proposition}
\begin{proof}
Put $M=2^L$ and consider
\[
 x_1=\tanh(\tfrac12x_1+s),\qquad
 x_2=\tanh(Mx_1+s),\qquad 0\le s\le1.
\]
At $s=0$ the unique root is zero, and the residual Jacobian and its
inverse are
\[
 H_x=\begin{pmatrix}1/2&0\\-M&1\end{pmatrix},\qquad
 H_x^{-1}=\begin{pmatrix}2&0\\2M&1\end{pmatrix}.
\]
Applying the inverse to $(1,0)$ gives norm at least $2M$. Hence this
family violates the polynomial-$K$ promise of
Theorem~\ref{thm:input-sweep} in the original Euclidean coordinates.
The first block has contraction factor $1/2$; compute it to
$b+L+O(1)$ bits in $O(b+L)$ iterations, then evaluate the second block.
Its propagated error is at most $M$ times the first error plus its own
evaluation error. This proves the polynomial bound for rational $s$ of
polynomial encoding length, uniformly over the sweep.
\end{proof}

The separation is between sufficient certificates, not an exponential
lower bound for every execution of the old tracker. It is also not a
separation from weighted-norm contraction methods: the norm
$\max\{|x_1|,|x_2|/(4M)\}$ makes the displayed layer map contractive
with factor at most $1/2$. Converting back to the requested state norm
costs $O(L)$ extra precision bits. Thus the contribution is explicit
state-accuracy and bit accounting, together with inclusion of the
previous continuation class, rather than a new contraction principle.

\subsection{Inheritance of quantitative solver guarantees}
\label{sec:solver-inheritance}
The scientific claim of this subsection is that certified continuation is
a single default mechanism for certified DEQ inference. First, an
instance certified for an established solver---a quantified contraction
regime in that solver's own coordinates---is solved by the one
continuation mechanism of Theorem~\ref{thm:solver-homotopy} with
polynomial bit cost, without requiring the original DEQ update to be
contractive in the Euclidean state norm. Second, the same mechanism then
strictly extends coverage beyond every inherited contraction class:
Section~\ref{sec:fold-continuation} continues a selected branch through
simple folds, where no state-update contraction certificate exists on
the route (Example~\ref{ex:tanh-fold}), and
Section~\ref{sec:automatic-continuation} generates the required local
certificates automatically. A small number of deliberate exceptions,
such as triangular ReLU systems solved by direct substitution, fall
outside the continuation route and are outlined explicitly below. None
of this asserts that every convergence theorem is a uniform polynomial
bit-complexity theorem.

\begin{theorem}[Continuation inherited from a contractive solver]
\label{thm:solver-homotopy}
Let $T$ be a contraction with factor $q=1-\gamma\in[0,1)$ on a
finite-dimensional normed vector space. Its fixed point must decode to the
requested equilibrium. Choose a known $y_0$ and a computable upper bound
$R\ge\max\{1,\norm{T(y_0)-y_0}\}$. Define
\[
 \Phi(z)=\frac{T(y_0+Rz)-y_0}{R},\qquad z=s\Phi(z),\quad 0\le s\le1.
\]
Every $s$ has a unique root $z_s$, with
\[
 z_0=0,\qquad \norm{z_s}\le\gamma^{-1},\qquad
 \norm{z_s-z_t}\le\gamma^{-2}|s-t|.
\]
Suppose decoding from $y$ is $C$-Lipschitz on the encountered region,
where $C\ge1$, and put
$p=b+\lceil\log_2(CR)\rceil+2$.
There is a rounded continuation producing output error at most $2^{-b}$
using
\[
 O\bigl(\gamma^{-3}(p+1)\bigr)
\]
elementary solver evaluations. This is polynomial bit complexity if
$\gamma^{-1}$, $p$, encoded trajectory sizes, and the costs of producing
or verifying the stated certificates, evaluating $\Phi$, and decoding
to the requested precision are polynomial in the input size and $b$.
\end{theorem}
\begin{proof}
Normalization preserves $q$ and gives $\norm{\Phi(0)}\le1$.
Banach's theorem applies uniformly to $s\Phi$. From the root equation,
$\norm{z_s}\le q\norm{z_s}+1$, and hence also
$\norm{\Phi(z_s)}\le\gamma^{-1}$. Subtracting the root equations gives
\[
 \norm{z_s-z_t}\le q\norm{z_s-z_t}
             + |s-t|\norm{\Phi(z_t)},
\]
which proves the motion bound without differentiability assumptions.

Set $r=1/2$, $N=\lceil4\gamma^{-2}\rceil$, and
$k_0=\lceil\gamma^{-1}\ln4\rceil$.
Consecutive roots move by at most $r/2$. A group of $k_0$ rounded
contraction steps, each with additive error at most $\nu$, reduces an
initial error $e$ to at most $e/4+\nu/\gamma$, since
$q^{k_0}\le e^{-\gamma k_0}\le1/4$.
Choose $\delta\le\min\{1/16,2^{-(p+2)}\}$ and $\nu\le\gamma\delta$.
The finite rounded tracker with $m=p+2$ such groups per stage stays in
its tracking balls and ends with normalized error at most $2^{-p}$.
Decoding amplifies this by at most $CR$; an additional output rounding
error at most $2^{-(b+1)}$ still leaves total error below $2^{-b}$.
There are $Nmk_0$ elementary calls. The bit-cost conclusion counts each
call at its required precision; analytic contraction by itself does not
supply that cost bound.
\end{proof}
The normalization, coordinate adapters, and source-specific bit backends
are separately supplied obligations. This construction preserves polynomiality,
not necessarily the original rate or polynomial degree. Intermediate
equations are auxiliary solver equations; their endpoint is the original
equilibrium. A supplied original certificate may instead use the original
zero-repair route directly.

\begin{corollary}[Explicit coverage of established solver classes]
\label{cor:published-coverage}
The following quantitative regimes are covered by the single continuation
mechanism of Theorem~\ref{thm:solver-homotopy}, subject to its specified
evaluation and encoding requirements.
\begin{enumerate}
\item \emph{Perron--Frobenius and block contractions.}
For the nonnegative Lipschitz comparison matrix $M$ of the CONE/BLIP
conditions of El~Ghaoui et al.~\cite{elghaoui2021}, a certificate
$v>0$, $Mv\le qv$, $q<1$, makes the update contractive in the weighted
maximum norm (or its block version). Polynomial $\gamma^{-1}$ and
polynomial bit length of the weights give the required regime.
For rational $M$ and a supplied rational $q>\rho(M)$, one may compute
$v=(qI-M)^{-1}\mathbf1$ by rational elimination and check positivity
and the inequality. Block norms and activations still require their own
certified evaluators. The triangular decomposition results are also
inherited through sequential block solves with the precision allocation
of Section~\ref{sec:block-extension}.
\item \emph{Monotone-operator DEQs.}
For the formulation $0\in F(x)+\partial f(x)$ of
Winston and Kolter~\cite{winston2020}, let $F$ be $m$-strongly monotone
and $L$-Lipschitz. Forward--backward splitting with $\alpha=m/L^2$
has factor $q\le\sqrt{1-m^2/L^2}$, so
$\gamma^{-1}\le2(L/m)^2$.
For their affine $F$, Peaceman--Rachford splitting in its auxiliary
variable has factor at most $\sqrt{1-m/L}$ when $\alpha=1/L$,
so $\gamma^{-1}\le2L/m$.
The proximal and resolvent evaluations and the decoder must have
polynomial bit cost; an arbitrary proximal oracle is not a free operation.
\item \emph{Non-Euclidean averaged contractions.}
The averaged update in the weighted maximum-norm theorem of
Jafarpour et al.~\cite{jafarpour2021} has
$q=1-\alpha(1-\operatorname{osL}(F))$ for its admissible step sizes.
At the stated maximal step, $\gamma^{-1}=\kappa_\infty$.
Polynomial effective condition number, encoded norm conversion, and
certified evaluation give coverage even when the raw DEQ update is not
a Euclidean contraction. The derivative-free theorem also permits
nonsmooth activations covered by that source's hypotheses.
\item \emph{Positive subhomogeneous DEQs.}
The positive $\mu$-subhomogeneous maps of
Sittoni and Tudisco~\cite{sittoni2024} contract in Thompson distance.
With $T(y)=\log F(\exp y)$ the norm is $\ell_\infty$ and $q=\mu$.
For their normalized maps use $q=2\mu<1$ in the general case, or
$q=\mu<1$ under their stronger positive-Jacobian hypotheses.
Polynomial inverse gap is required. Positivity, the encountered log
range, the exponential decoder's accuracy amplification, and the costs
of logarithm and exponential evaluation must also meet
Theorem~\ref{thm:solver-homotopy}; the metric change alone does not
establish a bit bound.
\item \emph{Architecturally contractive multiscale models.}
For a certified forward update with Lipschitz factor $q<1$, including
the contractive construction of Sato and Iiduka~\cite{sato2026}, the
same theorem applies with that factor and with the finite composition's
evaluation cost counted.
\end{enumerate}
\end{corollary}
\begin{proof}
Weighted norms turn comparison inequalities into ordinary contractions.
For forward--backward splitting, nonexpansiveness of the proximal map
and strong monotonicity give the squared factor
$1-2\alpha m+\alpha^2L^2$.
For the reflected affine resolvent, writing $a=F(x)-F(y)$ and $h=x-y$
gives
\[
 \norm{h-\alpha a}^2
 =\norm{h+\alpha a}^2-4\alpha\langle h,a\rangle
 \le\left(1-\frac{4\alpha m}{(1+\alpha L)^2}\right)
                         \norm{h+\alpha a}^2.
\]
The other reflection is nonexpansive. Finally,
$1-\sqrt{1-t}\ge t/2$ for $0\le t\le1$ yields the stated gaps.
The averaged-map rates and Thompson-metric contractions are the cited
results, with the latter becoming norm contractions under logarithmic
coordinates. Apply Theorem~\ref{thm:solver-homotopy} in each case.
\end{proof}

\paragraph{Outlined exceptions: direct solutions without a gap.}
Not every certified class is routed through the homotopy.
For triangular scalar ReLU systems the direct formula
$x_i=(b_i+\sum_{j<i}A_{ij}x_j)_+/(1-A_{ii})$, $A_{ii}<1$,
computes the equilibrium by substitution with polynomial rational cost,
including diagonal entries below $-1$; each step adds only polynomially
many fraction bits, and no contraction gap exists or is needed. This
exception is a polynomial certified solution in its own right and is
simply stated as such, not folded into
Theorem~\ref{thm:solver-homotopy}.

This list is an explicit coverage statement, not a claim about all past
or future DEQ results. For example, geometric convergence with an
unspecified instance-dependent factor, such as that established for
positive concave DEQs~\cite{gabor2024}, needs a quantitative bound on
that factor and its error prefactor before it yields uniform polynomial
complexity. An unquantified convergence result is not disproved or
excluded by an inconclusive run. Similarly, a noncontractive solver with
a separately established polynomial certified implementation is such an
exception; it need not fit the contractive homotopy. Absorption here
concerns inference, and does not transfer a source's training or
generalization guarantees.

\subsection{Continuation through simple folds}
\label{sec:fold-continuation}
The preceding inheritance theorem exploits an existing solver's
contraction. A different extension permits the input to reverse direction
while following a selected equilibrium branch. This changes the
continuation coordinate, not the equilibrium equation or the architecture.
Pseudo-arclength continuation and its regularity at simple folds are
classical; see Dickson et al.~\cite{dickson2006}. The issue here is the
quantitative certificate needed to combine that mechanism with finite
rounded tracking and budgeted termination.

Here $u$ is a scalar continuation coordinate; a prescribed vector-valued
input path may first be composed into $H$.
Write $z=(x,u)\in\mathbb R^{n+1}$, $H(u,x)\in\mathbb R^n$, and
$A=[H_x\ H_u]$. At a regular point of the zero curve let $t$ be an
oriented unit vector spanning $\ker A$. The phase-constrained Jacobian is
\[
 \mathcal B=\begin{pmatrix}A\\t^T\end{pmatrix}.
\]
\begin{lemma}[Bordered margin]
\label{lem:bordered-margin}
If $A$ has full row rank and smallest positive singular value at least
$\sigma>0$, then
$\norm{\mathcal B^{-1}}\le\max\{1,\sigma^{-1}\}$.
If a frozen-phase Jacobian differs from this one by at most $\eta$,
its minimum singular value is at least $\min\{1,\sigma\}-\eta$.
\end{lemma}
\begin{proof}
Decompose $v=v_\perp+a t$, with $v_\perp\perp t$. Then
$\norm{\mathcal Bv}^2=\norm{Av_\perp}^2+a^2$
is at least $\min\{1,\sigma^2\}\norm v^2$.
The perturbation statement follows from the triangle inequality.
Equivalently, for the normal right inverse $R=A^\dagger$,
$\mathcal B^{-1}(b,a)=Rb+at$.
\end{proof}

At a simple fold, $H_x$ has a one-dimensional kernel and cokernel.
If $v,w$ span them, the conditions
$w^TH_u\ne0$ and $w^TH_{xx}[v,v]\ne0$ imply a nondegenerate turning
point of the input projection. The first condition already makes $A$
full row rank: a left null vector of $A$ would be a multiple of $w$
and would also annihilate $H_u$. Thus $H_x^{-1}$ may cease to exist
while the bordered inverse remains bounded. Qualitative nonzero
conditions alone are insufficient for a uniform polynomial bound.

\begin{theorem}[Quantitative fold extension]
\label{thm:fold-extension}
Suppose a selected oriented $C^2$ zero-curve segment starts at a certified
root and reaches a certified terminal bracket for $u=u_*$.
Assume the following quantitative promises and backend contracts.
\begin{enumerate}
\item Its arclength is at most $L$, $A$ has full row rank with smallest
positive singular value at least $\sigma>0$, and $\norm{D^2H}\le M$
on a certified neighborhood of radius $\rho$.
\item In a terminal neighborhood the selected crossing is isolated by
the bracket and $|du/d\ell|\ge\beta>0$. The bracket specifies which
crossing is requested if the branch meets the input face more than once.
\item Local rounded evaluation, linear solving, oriented tangent
enclosure, and validation of overlapping phase charts and the terminal
bracket have polynomial bit cost at the precisions below. Their encoded
states have polynomial size. Validation uses robust margins, with a
polynomial bound on failed trials as well as successful trials.
\end{enumerate}
Put $K=\max\{1,\sigma^{-1}\}$ and choose
\[
 0<r\le\min\{1,\rho/4,[64K(1+M)]^{-1}\}.
\]
If $L,K,M,\rho^{-1},\beta^{-1},r^{-1}$ and the input/output encoding
overheads are polynomially bounded, selected-branch inference admits
polynomially many rounded continuation operations, including crossings
of simple folds. The number of local charts can be bounded by
$O(1+L/r)$; each has a uniformly bounded inverse and uses polynomial
precision and polynomially many corrector steps. The clocked version
returns a certified answer or inconclusive. With the original valid
certificate given priority, all original successful traces remain
available with zero repairs.
\end{theorem}
\begin{proof}
Parametrize the selected segment by arclength. Differentiating
$Az'=0$ and using $\langle z',z''\rangle=0$ gives
$\norm{z''}\le M/\sigma$. Thus the unit tangent changes by at most
$KM$ times the arclength traveled. On a forward arc of length $r$,
the frozen unit tangent $t_i$ has positive inner product, bounded away
from zero, with every current tangent. The phase
$a=t_i^T(z-z_i)$ is therefore a valid local increasing coordinate.
An arclength advance between $r/4$ and $r$ is obtained by choosing a
phase advance $r/2$; the same bounds hold with conservative slack for
sufficiently accurate enclosures of the center and tangent.

At the center Lemma~\ref{lem:bordered-margin} gives inverse bound $K$.
On a ball of radius $2r$, the first row block changes by at most $2Mr$;
choose the encoded tangent error at most $1/(16K)$. The perturbed
border then has inverse norm at most $2K$. The augmented equations
\[
 H(z)=0,\qquad t_i^T(z-z_i)-a=0
\]
have derivative in $a$ of norm one. Their local root velocity is at
most $2K$, and their derivative variation is bounded by $M$.
Consequently the local contraction radius, phase mesh, and rounding
precision can be chosen by the same inequalities as the finite rounded
tracker. One can also use rounded Newton correction: Taylor's theorem
bounds its exact error by $KM e^2$ when the inverse is bounded by $2K$,
which is at most $e/4$ on these tracking balls; rounding adds $\delta$.
Each chart therefore needs only polynomially many corrections at
precision $b$ plus logarithmic condition and encoding overheads.

Certified overlaps pass the approximate root and its error enclosure
to the next chart on the same oriented local branch. Every completed
nonterminal chart advances arclength at least $r/4$, so at most
$1+4L/r$ charts are required, up to the fixed slack used for enclosures.
The local geometric argument establishes available chart size;
the assumed backend contract accounts for actually validating and
encoding those charts. It is not a free global certificate oracle.

At the terminal face replace the tangent row by the input projection
$e_u^T$. In the decomposition $v=Rb+a t$, a terminal right-hand side
$c$ requires
\[
 a=\frac{c-e_u^TRb}{e_u^Tt},\qquad
 \norm v\le K\norm b+\beta^{-1}(|c|+K\norm b).
\]
This is a polynomial inverse bound. The isolated bracket and its
certified neighborhood permit final rounded correction to the requested
accuracy and prevent choosing a different crossing. Summing the chart,
validation, and terminal costs proves polynomial work. The budgeted
wrapper and original-trace preservation follow from
Theorem~\ref{thm:budgeted-certification}.
\end{proof}

The length bound and terminal bracket are essential: regularity at every
fold does not imply a short branch or that it reaches a specified input.
A terminal point which is itself a fold has $\beta=0$ and needs a
different terminal certificate. Rank loss of the full matrix $A$, as at
certain branch points, is also outside this theorem. Coordinate changes
may be counted as continuation repairs; they do not add training channels.

\paragraph{Containment of the original inference class.}
On an original unit-interval input sweep with
$\norm{H_x^{-1}}\le K_0$ and $\norm{H_u}\le B_0$, the smallest
positive singular value of $[H_x\ H_u]$ is at least $K_0^{-1}$.
Its graph has length at most $\sqrt{1+(K_0B_0)^2}$, and its unit tangent
satisfies $|du/d\ell|\ge[1+(K_0B_0)^2]^{-1/2}$.
Thus the old geometric promises imply the new ones after polynomial
rescaling. More directly, the dispatcher retains the original backend:
membership in the old certified class never forces a fold repair or
new chart-validation obligation. The extended accepted class is a union,
not a replacement of the original one.

\begin{example}[A tanh equilibrium curve with two simple folds]
\label{ex:tanh-fold}
Consider $H(u,x)=x-\tanh(2x+u)$. Its zero curve is
\[
 u(x)=\operatorname{atanh}(x)-2x,\qquad -1<x<1.
\]
On this curve,
\[
 H_x=2x^2-1,\quad H_u=-(1-x^2),\quad
 H_{xx}=8x(1-x^2).
\]
At $x=\pm1/\sqrt2$, $H_x=0$, $H_u=-1/2$, and
$H_{xx}=4x\ne0$: both are simple folds. The original input-sweep
certificate fails there. With phase $x$, however, the border and its
inverse at either fold are
\[
 \begin{pmatrix}0&-1/2\\1&0\end{pmatrix},\qquad
 \begin{pmatrix}0&1\\-2&0\end{pmatrix}.
\]
The Euclidean inverse norm is only $2$.

On the whole segment $|x|\le4/5$, one has
$|H_u|\ge9/25$, $|H_x|\le1$, and
$|u'(x)|=|(1-x^2)^{-1}-2|\le1$.
The length is at most $(8/5)\sqrt2$. The $x$-phase inverse sends
$(h,a)$ to $(a,(H_xa-h)/(1-x^2))$ and has maximum-norm operator
bound $50/9<6$. Thus a single choice of phase crosses both folds
with constant geometric bounds. The Hessian of $H$ has norm at most
$10$, so sufficiently small neighborhoods retain a constant margin.
Certified tanh and rational arithmetic provide polynomial evaluation
cost in this fixed-dimensional example.

For a concrete selected terminal crossing take $u_*=-51/100$ and
$x\in[3/4,4/5]$. The endpoint values satisfy
$u(3/4)<-51/100<u(4/5)$, and $u'(x)\ge2/7$ throughout the bracket.
For instance $1.94<\log7<1.95$ and $1.098<\log3<1.099$
certify these strict inequalities. Such bounds follow from the positive
series $\log a=2\sum_{j\ge0}v^{2j+1}/(2j+1)$,
$v=(a-1)/(a+1)$, whose remainder after $N$ terms is at most
$2v^{2N+1}/((2N+1)(1-v^2))$; $N=32$ suffices for both displayed
intervals. Start at $x=-4/5$, $u=8/5-\log3$, and orient the curve toward
increasing $x$. The initial root is exact in this expression encoding;
the logarithm is evaluated with certified rounding. The selected
terminal bracket lies beyond both folds and supplies a uniform
transversality bound $\beta\ge2/(7\sqrt2)$.
Hence this is a branch-selection problem to which the fold theorem
applies with polynomial cost in $b$, despite singular $H_x$ along the
route. It is a concrete strict extension of the original regular-sweep
promise, not a claim that this scalar equation was previously hard.

\end{example}

\begin{corollary}[Combined inference guarantee]
\label{cor:combined-inference}
Fix the original certified backend, a finite set of quantitative solver
adapters from Section~\ref{sec:solver-inheritance}, and a fold backend
satisfying Theorem~\ref{thm:fold-extension}. Require every accepted output
to satisfy the same requested equilibrium and branch-selection predicate.
With budgets dominating the corresponding polynomial work bounds, their
dispatcher succeeds on the union of these promise classes in polynomial
bit time. A valid original certificate retains its original zero-repair
trace. The fold component strictly enlarges the original regular-sweep
promise, as witnessed by Example~\ref{ex:tanh-fold}.
\end{corollary}
\begin{proof}
Apply the budgeted dispatcher of Theorem~\ref{thm:budgeted-certification}
to the three routes.
The example has a certified selected route through points where $H_x$
is singular, so that route violates the original sweep hypothesis.
This is strict containment of certified path problems; it is not a
lower bound against other algorithms for the same terminal equation.
\end{proof}

\subsection{A finite encoded fold tracker}
\label{sec:encoded-fold-tracker}
The geometric extension above now has an executable finite tracking
layer. Its input is a finite list of encoded phase charts. Each chart
provides a Boolean validation function, an encoded rounded update, an
iteration count, a coordinate-repair flag, and cost functions for
validation and updating. Roots and analytic certificates occur only in
the correctness proof. The validator's soundness is a contract;
arbitrary Boolean acceptance is not a certificate.

The runner charges one tick per validation and correction, and one final
tick for successful reporting. It checks each chart before use and
executes its correction block only if the block fits the remaining fuel.
Failed validation or insufficient fuel returns inconclusive, charging
the work already performed. Exhausting the budget at an intermediate
state does not produce a certified answer.

\begin{theorem}[Encoded fold-tracker correctness and bounded work]
\label{thm:encoded-fold-tracker}
Let $d$ decode encoded states into a normed state space. A supplied plan
has roots $q_0,\ldots,q_N$, radius $r\ge0$, and rounding allowance
$0\le\delta\le r/8$. Assume
$\norm{d(c_0)-q_0}\le r/2$ and these chart contracts:
\begin{enumerate}
\item Adjacent roots satisfy $\norm{q_i-q_{i+1}}\le r/2$.
On acceptance of chart $i$ from a valid warm start, its encoded update
$U_i$ satisfies, for every encoding $c$ in the target ball,
\[
 \norm{d(U_i(c))-q_{i+1}}
 \le\tfrac14\norm{d(c)-q_{i+1}}+\delta.
\]
\item Every nonterminal chart performs at least one correction.
The final chart has $m_N$ corrections satisfying
$2^{-m_N}r+2\delta\le\varepsilon$.
Its root satisfies a predicate $\operatorname{Selected}$ including the
equilibrium equation, terminal input, and requested branch.
\end{enumerate}
Every successful answer $c$ satisfies $\operatorname{Selected}(q_N)$
and $\norm{d(c)-q_N}\le\varepsilon$. All corrections stay inside
their certified target balls. Independently of success, charged ticks
and coordinate repairs are at most the fuel budget $P$.

If every validation succeeds on the actual corrected trace, the runner
succeeds with
\[
 P_* = 1+\sum_{i=1}^{N}(1+m_i).
\]
Its repair count is exactly the sum of the chart repair flags, and is
zero when all flags are false. If $m_i\le m$ and every validation and
correction has modeled work at most $C\ge1$, including failed checks,
then
\[
 W_{\rm run}\le P C,\qquad P_*\le N(m+1)+1.
\]
Thus fuel $N(m+1)+1$ suffices on the successful-validation class.
Polynomial bounds on $N,m,C$ give polynomial work for this finite layer.
For a bit bound, primitive cost functions must cover encoded arithmetic;
initialization, certificate construction, and output serialization must
have separately counted polynomial budgets.
\end{theorem}
\begin{proof}
The overlap bound turns an $r/2$ warm start for $q_i$ into an $r$ warm
start for $q_{i+1}$. Induction on actual encoded corrections gives
$e_k\le r$ and $e_k\le2^{-k}r+2\delta$. After one or more
corrections the error is at most $r/2$, preserving the next warm-start
condition. For the final chart use its finer error bound, rather than
requiring $r/2\le\varepsilon$. Induction over the validated chart list
proves selected-terminal accuracy. The statement also permits an
empty plan when the initial certified enclosure is already accurate.

The program structurally recurses over the chart list. Each validation
consumes one tick, a performed correction block consumes its iteration
count, and successful reporting consumes one final tick. Failed checks
consume their validation cost. Induction gives the fuel and work bounds
even on inconclusive runs. A successful-trace induction proves exact
reproduction of corrected states, ticks, and repair counts whenever fuel
covers $P_*$. Summing the iteration bounds proves the uniform budget.
\end{proof}

\paragraph{Connection to bordered Newton correction.}
For phase functional $\lambda$, the chart equation is
$\mathcal F(z)=(H(z),\lambda(z-z_i)-a)$.
Its root identity and the Newton estimate hold even
when state and residual spaces have different types. For a bordered derivative
$J$ and left inverse $R$ with $\norm R\le K$, Taylor remainder bound
$H_2\norm{z-q}^2$, solve residual at most $\eta$, and rounding error
at most $\zeta$, the decoded encoded update satisfies
\[
 \norm{z_{\rm new}-q}
 \le K H_2\norm{z-q}^2+K\eta+\zeta.
\]
Thus $KH_2r\le1/4$ and $K\eta+\zeta\le\delta$ imply the chart
contraction contract. The hypothesis connects the actual encoded update
to its step and rounding error; it does not identify an arbitrary
numerical routine with exact Newton iteration.

\paragraph{Precision without exponentially small charts.}
Keep the geometric radius $r$ independent of $b$. If $r\le2^\ell$,
choose $m_N=b+\ell+2$ and
$\delta\le\min\{r/8,2^{-(b+2)}\}$.
The checked precision lemma gives $\varepsilon=2^{-b}$; only precision
and correction counts grow with $b$. Internal charts need only maintain
their $r/2$ warm-start accuracy. Shrinking the entire chart radius to
$2^{-b}$ would instead risk exponentially many charts.

\paragraph{An executable fold-crossing instantiation.}
An executable example instantiates the runner and its semantic
theorem for $H(x,u)=x^2+u$. With phase $x=a$, exact rational Newton is
\[
 (x,u)\longmapsto(a,x^2-2ax).
\]
In the product maximum norm its new error from $(a,-a^2)$ is $(x-a)^2$,
at most one quarter of the previous error on the radius-$1/4$ ball.
This holds for arbitrary rational encoded states, not just the
sample trace.

Starting at $(-1/8,-1/64)$, phase targets $-1/16,0,1/16,1/8$ and
correction counts $2,2,2,10$ cross the simple fold at $(0,0)$.
The selected endpoint is the positive root at input $u=-1/64$, although
the negative start root has that same input. The generic soundness
theorem certifies $2^{-8}$ accuracy; exact evaluation returns
$(1/8,-1/64)$. The run uses 21 ticks and one coordinate repair. Fuel 20
yields inconclusive, and an invalid initial enclosure is rejected with
its validation charged. Test costs of three abstract units per check
and five per update give 93 units including reporting. These constants
test accounting; they are not measured rational-arithmetic bit costs.

The remaining boundary is general geometric certificate construction
and its bit-cost analysis, including tangent and terminal-bracket
generation. The finite encoded tracker does not establish those
separate backend promises.

\subsection{Automatic local certificates and selected-path assembly}
\label{sec:automatic-continuation}
We now give a construction of local certificates in arbitrary
finite dimension, followed by a conditional automatic assembly theorem.
The input does not include a list of phase charts. It does include a
certified seed, an orientation, quantitative neighborhood bounds, and a
terminal-selection contract. These distinctions are essential: producing
local certificates is different from discovering an initial equilibrium,
bounding the length of its branch, or deciding which endpoint is wanted.
Parametrized contraction validation and continuation through folds are
established techniques; see \cite{berg2021validated,dickson2006}.
Here we specify their rational tests, search overhead, and connection to
the budgeted DEQ tracker.

\paragraph{A rational local test.}
Use the maximum norm and its induced matrix norm throughout this
subsection. Let $G:\mathbb R^d\to\mathbb R^d$ be continuously
differentiable on a neighborhood of
$B=\{z:\norm{z-c}_\infty\le r\}$, with $c\in\mathbb Q^d$
and rational $r>0$.
Choose a rational nonsingular matrix $R$. Certified enclosures supply
\[
 Y\ge\norm{RG(c)}_\infty,\qquad
 q\ge\sup_{z\in B}\norm{I-RDG(z)}_\infty.
\]
For example, an entrywise enclosure $DG(B)\subseteq\widehat J+[-\Delta,\Delta]$
gives the computable rational bound
\[
 q=\norm{I-R\widehat J}_\infty+
       \norm{|R|\Delta}_\infty.
\]
No sampling of the interior of the box is used.

\begin{lemma}[Sound rational box certificate]
\label{lem:auto-box}
If $q\le1/4$ and $Y\le r/8$, then $G$ has exactly one zero in $B$.
Its distance from $c$ is at most $Y/(1-q)\le r/6$.
The map $T(z)=z-RG(z)$ contracts distances by at most $q$ on $B$.
\end{lemma}
\begin{proof}
The mean-value estimate gives the contraction bound. Also
$\norm{T(z)-c}_\infty\le qr+Y\le3r/8$, so $T$ maps the complete
closed box into itself. Banach's theorem gives a unique fixed point;
nonsingularity of $R$ identifies fixed points with zeros of $G$.
The fixed-point equation gives the stated a posteriori distance bound.
\end{proof}
An accepted test with $q<1$
already certifies nonsingularity of the preconditioner, so no separate
determinant check is charged.

\begin{lemma}[Local automatic generation]
\label{lem:auto-generation}
Suppose $A=DG(c)$ is invertible, $\norm{A^{-1}}_\infty\le K$, $K\ge1$,
and an effective interval evaluator has derivative variation bound
$Mr$ on every radius-$r$ box about $c$ with $r\le\rho$.
More precisely, it encloses $DG(c)$ to operator-norm error $\epsilon_J$
and the box variation by $Mr$, with polynomial bit cost in the input
length and requested precision. It encloses $G(c)$ to error
$\epsilon_F$ with the same cost property. Define
\[
 r_* =\min\{1,\rho,(64K\max\{1,M\})^{-1}\}.
\]
If $\norm{A^{-1}G(c)}_\infty\le r_*/128$, an automatic rational
search finds a certificate of Lemma~\ref{lem:auto-box}. It needs only
polynomially many trials in the input length, $\log K$, and
$\log(1/r_*)$, provided the evaluator's encoding overhead is polynomial
in these quantities. This assertion does not assume a zero in advance.
\end{lemma}
\begin{proof}
Choose a dyadic $r\in[r_*/2,r_*]$, an approximation $\widehat A$ with
$\epsilon_J\le1/(64K)$, and set $R=\widehat A^{-1}$ using exact
rational elimination. The Neumann estimate gives $\norm R_\infty\le2K$
and $\norm{RA}_\infty\le2$. The center defect and box variation give
$q\le2K(\epsilon_J+Mr)\le1/16$.
With $\epsilon_F\le r_*/(128K)$, the residual enclosure satisfies
\[
 Y\le2\norm{A^{-1}G(c)}_\infty+2K\epsilon_F
 \le r_*/32\le r/16.
\]
There is therefore strict slack in both acceptance tests.
At stage $j$, test all radii $2^{-k}$, $0\le k\le j$, at precision
$j$, skipping singular approximate matrices. Revisit earlier radii as
precision increases. An adequate stage is polynomial in the stated
logarithmic quantities, and there are $O(j^2)$ trials up to that stage.
Exact rational elimination and the matrix tests have polynomial bit cost.
Blindly shrinking a box at a fixed center would not prove completeness:
it can eventually exclude the desired zero.
\end{proof}

\paragraph{Explicit enclosures for certified smooth activations.}
The enclosure layer consumes an activation interface rather than $\tanh$
identities. Call a scalar activation \emph{certified} if it supplies a
$C^1$ value $\varphi$ with derivative $\psi=\varphi'$, a Lipschitz
constant $M_2$ for $\psi$ on the encountered preactivation range, and a
rational evaluator computing $\varphi$ and $\psi$ at rational arguments
to accuracy $2^{-p}$ with polynomial bit cost on that range. For
rational data consider
\[
 F(x,u)=x-\varphi(Wx+vu+a),\qquad d=n+1,
 \qquad s_i=\sum_j|W_{ij}|+|v_i|,
\]
with certified activations acting componentwise; heterogeneous per-row
activations are permitted, and an affine phase row is the zero
activation. The variation of Jacobian row $i$ on a radius-$r$ box has
absolute row sum at most $M_{2,i}s_i^2r$. For a phase row with constant
derivative, a sharper preconditioned bound is
\[
 M_R=\max_k\sum_{i=1}^n |R_{ki}|M_{2,i}s_i^2,
 \qquad q\le q_{\rm center}+M_Rr+\text{enclosure error}.
\]
These bounds require polynomially many rational operations. Large
coefficients can make the resulting chart radius small; polynomial
encoding length alone does not bound the number of charts.
For $\tanh$, $|\tanh'|\le1$ and $|(1-\tanh^2)'|\le2$ give $M_2=2$ and
recover the constants $2s_i^2r$. The logistic sigmoid inherits
$M_2=1/4$ through the exact identity $\sigma(y)=(1+\tanh(y/2))/2$.
The exact Gaussian error
linear unit $\operatorname{gelu}(y)=y\Phi(y)$ has second derivative
$(2-y^2)$ times the normal density, and the elementary bounds
$|2-y^2|e^{-y^2/2}\le2$ and $\sqrt{2\pi}\ge2$ give $M_2=1$.
Their evaluators are supplied by
Proposition~\ref{prop:activation-evaluators} below. Softplus and other
smooth activations with bounded
second derivative satisfy the interface with their own constants and
tail certificates; activations with derivative kinks, such as ReLU,
fall outside this lemma and belong to the direct route outlined in
Section~\ref{sec:solver-inheritance}.
Certified tanh evaluation on arbitrary rational arguments can use
saturation when $|y|\ge p+2$, since
$|\tanh y-\operatorname{sign}y|\le2e^{-2|y|}$ and
$1-\tanh^2y\le4e^{-2|y|}$. On the remaining interval $|y|=O(p)$,
the exponential series with a certified remainder and $O(p)$ terms
(with a sufficiently large constant) yields polynomial bit complexity.
This extends the implemented bounded evaluator, whose certified bounds
cover $|y|\le4$. The center enclosures remain hypotheses discharged by
each activation's evaluator.

\begin{proposition}[Sigmoid and GELU evaluation]
\label{prop:activation-evaluators}
(i) For rational $u$ with $|u|\le8$ and precision $p$, the rational
values $\widehat\sigma_p(u)=(1+\widehat t_p(u/2))/2$ and
$\widehat\sigma'_p(u)=\widehat d_p(u/2)/4$ built from the evaluator of
Proposition~\ref{prop:concrete-tanh} satisfy
\[
 |\widehat\sigma_p(u)-\sigma(u)|\le2^{-(p+3)},\qquad
 |\widehat\sigma'_p(u)-\sigma'(u)|\le2^{-(p+3)},
\]
with encoded fraction-size and work envelopes equal to those of
Proposition~\ref{prop:concrete-tanh} up to fixed constants.
(ii) For rational $y$ and precision $p$, a rational $\widehat g_p(y)$
with $|\widehat g_p(y)-\operatorname{gelu}(y)|\le2^{-p}$ is computable
with bit cost polynomial in $p$ and the encoding length of $y$.
\end{proposition}
\begin{proof}
(i) The identity $\sigma(u)=(1+\tanh(u/2))/2$ halves the tanh value
error and quarters the derivative error; $|u|\le8$ places $u/2$ in the
certified box, and the transform costs one addition, one constant
inversion, and one multiplication beyond the bounded tanh evaluation.
(ii) Write $\operatorname{gelu}(y)=y\Phi(y)$ with
$\Phi(y)=\tfrac12+\tfrac1{\sqrt{2\pi}}\int_0^ye^{-t^2/2}\,dt$.
If $y^2\ge2(p+2)\ln2$ and $|y|\ge1$, the tail bound
$\int_{|y|}^\infty e^{-t^2/2}\,dt\le e^{-y^2/2}/|y|$ together with
$\sqrt{2\pi}\ge2$ certifies saturation: $\operatorname{gelu}(y)$
differs from $y$ (for $y>0$) or from $0$ (for $y<0$) by at most
$2^{-(p+2)}$, so the saturated value is acceptable and only
$|y|=O(\sqrt{p+1}\,)$ remains. There, termwise integration of the
exponential series gives the alternating series
\[
 \int_0^ye^{-t^2/2}\,dt=\sum_{k\ge0}
 \frac{(-1)^k\,y^{2k+1}}{2^k\,k!\,(2k+1)},
\]
whose term magnitudes halve once $2(k+1)\ge y^2$; the remainder after
$O(p)$ terms (with a fixed constant) is therefore below the first
omitted term and at most $2^{-(p+4)}$, by the factorial-decay argument
of Proposition~\ref{prop:concrete-tanh}. A rational enclosure of
$1/\sqrt{2\pi}$ to $p+O(1)$ bits follows from a certified rational
enclosure of $\pi$ and integer square-root extraction, both classical
and polynomial. Multiplying the enclosures and by $y$, and summing the
error contributions weighted by $|y|=O(\sqrt{p+1}\,)$, gives accuracy
$2^{-p}$ after fixed constant adjustments. All expression heights are
polynomial in $p$ and the input length, so the integer-fraction
envelope of Proposition~\ref{prop:concrete-tanh} applies verbatim.
\end{proof}

\begin{lemma}[A finite choice of phase coordinates]
\label{lem:auto-phase}
Let $A=DF(z)\in\mathbb R^{n\times(n+1)}$ have full row rank and smallest positive
singular value at least $\sigma>0$. At least one of the $d=n+1$ borders
$B_j=(A;e_j^T)$ is invertible and satisfies
\[
 \norm{B_j^{-1}}_\infty\le3d\max\{1,\sigma^{-1}\}.
\]
\end{lemma}
\begin{proof}
A Euclidean unit vector $t$ spanning $\ker A$ has a component
$|t_j|\ge d^{-1/2}$. The solution of $Av=b$, $v_j=s$ is
$v=A^\dagger b+t(s-e_j^TA^\dagger b)/t_j$.
For $\norm{(b,s)}_\infty\le1$, its norm is at most
$\sigma^{-1}\sqrt n+\sqrt d(1+\sigma^{-1}\sqrt n)$,
which is bounded by the displayed quantity.
\end{proof}
Trying all coordinate phases costs $d$ trials per search level, rather
than a grid of directions. A simple fold of the input projection does
not obstruct this test when $DF$ retains full row rank.

\paragraph{Validating a whole chart family.}
For a chosen signed coordinate functional $\lambda=\pm e_j^T$, use
\[
 G_s(z)=(F(z),\lambda(z-c)-s).
\]
Its derivative is independent of $s$. If the same box test gives
$q\le1/4$ and $Y_0\le r/16$, then it validates every $|s|\le h$ with
\[
 h=\frac{r}{16\max\{1,\norm{Re_d}_\infty\}}.
\]
Indeed $Y_s\le Y_0+|s|\norm{Re_d}_\infty\le r/8$.
The unique roots $z(s)$ form a continuously differentiable curve by the
implicit-function theorem, and
\[
 \norm{z'(s)}_\infty\le(1-q)^{-1}\norm{Re_d}_\infty,
 \qquad \lambda z'(s)=1.
\]
Thus the certificate describes a connected arc, not just isolated zeros.
Derivative enclosures also certify the sign of a previous phase applied
to $z'(s)$ at an overlap, fixing the orientation after a coordinate switch.
A quantitative arc bound
$(1-q)\norm{z(s)-z(s')}_\infty\le|s-s'|\,\norm{Re_d}_\infty$ also
follows from a contraction argument in place of the differentiated
identity.

\begin{theorem}[Automatic assembly of a selected continuation path]
\label{thm:auto-assembly}
Let a regular oriented zero-curve segment of the rational DEQ with
certified activations, $\tanh$ in particular, start
at a supplied certified seed. Suppose the following quantitative promises
hold; the segment itself and a chart list are not algorithmic inputs.
\begin{enumerate}
\item The segment has Euclidean length at most $L$ and lies in a
neighborhood with radius $\rho$, full-row-rank margin $\sigma$, and
effective derivative enclosure constant $M$. Coordinates, rational
data, and evaluator queries have polynomial encoding length.
\item A sound terminal monitor identifies the requested crossing of
$u=u_*$ on this oriented branch. It has polynomial cost and a positive
detection margin $\tau$: forward arcs of length at most $\tau$ cannot
skip the selected terminal neighborhood without a certified bracket.
Inside that bracket the crossing is unique and $|du/d\ell|\ge\beta>0$.
The monitor can refine enclosures to settle its decisions with polynomial
precision. These are explicit terminal-selection promises, not consequences
of local regularity.
\item $L,\rho^{-1},\sigma^{-1},M,\tau^{-1},\beta^{-1}$ and the
complete primitive budgets are polynomially bounded in the instance
length. The seed enclosure is refinable with polynomial cost.
\end{enumerate}
There is a rational adaptive algorithm that generates local box
certificates, assembles overlapping oriented charts automatically, and
returns the selected equilibrium to accuracy $2^{-b}$ with polynomial
bit cost. No global contraction of the DEQ state update is required.
With finite fuel it returns either certified success or inconclusive.
The original inference backend can retain priority and its zero-repair
trace on every instance already covered by its certificate.
\end{theorem}
\begin{proof}
Set $K_0=3d\max\{1,\sigma^{-1}\}$.
At each certified root, try the
$d$ coordinate borders; Lemma~\ref{lem:auto-phase} provides one with
inverse bound $K_0$. Shrink a common dyadic radius inside the promised
neighborhood so that $K_0Mr$ is smaller than a fixed constant, for
example $1/256$. Refining the current root to a rational center costs
only $O(\log(1/r)+\log K_0)$ extra bits. Perturbation then preserves
a border bound $2K_0$, and further refinement makes its Newton
displacement as small as required by Lemma~\ref{lem:auto-generation}.
The local search therefore finds a certificate with strict slack.
Use the same tests with $Y_0\le r/16$ to construct the parametric family.

At a transition retain the endpoint enclosure of the previous family.
Make it strictly contained in the next box. This enclosure contains an
actual zero $p$ on the previously selected arc. Its phase value
$s_p=\lambda(p-c)$ is enclosed rationally. Arrange
$|s_p|\le h/8$ by refining that enclosure and the center.
The next family's uniqueness identifies $p=z(s_p)$; it cannot select
an unrelated zero in the box. Orient the new coordinate using a
certified sign of the old phase applied to the new tangent, reversing
the sign when necessary. This sign is nonzero at the shared root because
both phases parametrize its one-dimensional tangent. Its quantitative
margin follows from the bounds on the two bordered inverses: the old
unit-tangent component is at least $1/(4\sqrt d K_0)$, allowing
for variation inside the previous box.
Refinement therefore decides the sign with polynomial precision.

Advance to phase $s=h/2$. From $|s_p|\le h/8$, the forward coordinate
displacement is at least $3h/8$, so the traversed arclength is at least
$3h/8$. It is at most a fixed multiple of $\sqrt d K_0h$ by the
derivative bound for the family. Cap $h$ further to make this upper
bound at most $\tau$. Uniform neighborhood and conditioning promises
permit a common inverse-polynomial lower bound $h_{\min}>0$ for
these accepted advances. The interval certificates have strict slack;
dyadic radius and precision searches reach that scale after polynomially
many trials, including failures. Thus before terminal detection there
are at most $1+\lceil8L/(3h_{\min})\rceil$ advances. Shared-root
uniqueness and the oriented phase changes prove inductively that every
advance lies on the branch chosen by the seed and orientation.

Apply the terminal monitor to the validated arcs. Its detection contract
gives a bracket for the selected crossing before an advance can skip it.
The transverse terminal equation $(F(z),u-u_*)=0$ has an inverse bounded
polynomially in $\sigma^{-1}$ and $\beta^{-1}$, by the tangent/normal
decomposition used in Lemma~\ref{lem:auto-phase}. Refine the bracket
into a validated terminal box, then use rounded contraction corrections
to reach $2^{-b}$ accuracy. Internal radii and progress bounds are
independent of $b$; terminal correction counts and precision grow
polynomially in $b$. Rational rounding to the certified error budget
prevents unbounded denominator growth. Sum matrix, evaluation, search,
overlap, monitoring, correction, and output costs to obtain the bit bound.
The budgeted dispatcher charges failed attempts as well, and preserving
the old backend first gives zero-repair inclusion as before.
\end{proof}

\paragraph{What is automatic, and what remains supplied.}
The theorem constructs the chart list, local root enclosures, phase
changes, and overlap witnesses. It does not infer a global length bound
or terminal-selection promise from arbitrary weights. A supplied rational
terminal region with certified entry margins and a transverse unique
crossing is one way to implement the monitor; automatic discovery of
such a region is a separate problem. Without a positive progress scale,
even inexpensive local certificates may require exponentially many
charts. Full rank loss of $DF$, an uncertified seed, or an ambiguous
terminal event is outside the completeness promise.

\section{A concrete rational backend}
\label{sec:concrete-backend}
The numerical contracts can be instantiated constructively for a bounded
activation evaluator and a scalar recurrent DEQ. This example separates
the cost of representing numerical data from the number of iterations.
It does not assume an exact transcendental evaluation oracle.

\begin{proposition}[Quantitative rational activation evaluation]
\label{prop:concrete-tanh}
For rational $z$ with $|z|\le4$ and requested precision $p\in\mathbb N$,
the Taylor order $n=p+132$ makes the exponential enclosure at $2|z|$
acceptable. Transform its endpoints through $(e-1)/(e+1)$, take their
midpoint, and restore the sign of $z$. The resulting rational value
$\widehat t_p(z)$ and derivative $\widehat d_p(z)=1-\widehat t_p(z)^2$
satisfy
\[
 |\widehat t_p(z)-\tanh z|\le2^{-(p+2)},\qquad
 |\widehat d_p(z)-(1-\tanh^2z)|\le2^{-(p+1)},\qquad
 |\widehat t_p(z)|\le1,\quad 0\le\widehat d_p(z)\le1.
\]
For an input fraction whose numerator and denominator have magnitude at
most $2^S$, an unreduced integer-fraction implementation returns a value
fraction with both magnitudes at most $2^{H_p(S)}$, where
\[
 H_p(S)=C_{\rm enc}(p+1)^3(S+1).
\]
Its binary-arithmetic work envelope is at most
$C_{\rm eval}(H_p(S)+1)^3$. The positive constants $C_{\rm enc}$ and
$C_{\rm eval}$ depend only on the fixed activation box and arithmetic
model, and are independent of $S$ and $p$.
\end{proposition}
\begin{proof}
For $|v|\le8$, the term $|v|^{32}/32!$ is at most $2^{96}$.
Every subsequent term is at most half the preceding one. Thus the
exponential remainder radius at order $p+132$ is at most $2^{-(p+3)}$.
At the nonnegative argument $2|z|$, the Taylor sum is at least one, so
the lower exponential endpoint is positive. The map $(e-1)/(e+1)$
has Lipschitz constant at most two on $[0,\infty)$; the midpoint error
is consequently at most twice the exponential radius. Squaring values
in $[-1,1]$ amplifies their error by at most two.

The encoding implementation uses signed integer fractions without gcd
normalization. Multiplication adds the numerator and denominator size
envelopes; addition adds them and charges one further bit. A straight-line
expression of height $h$ has fraction size envelope at most $h$ and
work envelope $O((h+1)^3)$. The expression for the $j$th Taylor term
has height $O(j(S+1)+(j+1)^2)$: powers contribute linearly in $jS$,
and factorial products contribute quadratically in $j$. Summing $n$
terms gives height $O(n^2(S+1)+n^3)$; the endpoint transforms change
this by a constant factor. Since $n=p+132=O(p+1)$, fixed constants
give the stated size and work bounds.
\end{proof}

\begin{theorem}[A finite encoded noncontractive DEQ solver]
\label{thm:concrete-deq}
For rational $u$ with $|u|\le1$, the equilibrium equation
\[
 x=\tanh(-2x+u)
\]
has a unique root $q\in[-1,1]$. Given $b\in\mathbb N$, set $p=b+4$,
$x_0=0$, and compute $b+2$ updates
\[
 x_{k+1}=\operatorname{Round}_p\!\left(
       \frac{x_k+\widehat t_p(-2x_k+u)}2\right),\qquad
 \operatorname{Round}_p(y)=2^{-p}\lfloor2^py\rfloor.
\]
Every stored state belongs to $[-1,1]$ and has a dyadic representation
with numerator magnitude at most $2^p$ and denominator $2^p$.
The final state satisfies $|x_{b+2}-q|\le2^{-b}$.
If the input has fraction size envelope $S$, there is a constant
$C_{\rm run}>0$, independent of $S$, $b$, and $u$, such that the
encoded run's work envelope satisfies
\[
 W(S,b)\le C_{\rm run}(b+1)(S+b+1)^{12}.
\]
In particular, $W(S,b)=O((S+b+1)^{13})$. These are conservative
upper bounds for this implementation, not optimality claims.
\end{theorem}
\begin{proof}
For $H_u(x)=x-\tanh(-2x+u)$, one has $1\le H_u'(x)\le3$.
The endpoint signs at $-1$ and $1$ give existence, and strict
monotonicity gives uniqueness. Moreover, $|x-q|\le|H_u(x)|$.
The damped map $T_u(x)=(x+\tanh(-2x+u))/2$ has derivative in
$[-1/2,1/2]$, hence contraction factor $1/2$. Its rational approximation
preserves the state box. Downward dyadic rounding also preserves this
box and has error at most $2^{-p}$. The complete update error is at
most $2\,2^{-p}$, giving by induction
\[
 |x_k-q|\le2^{-k}+4\,2^{-p}.
\]
The stated $k$ and $p$ give the required accuracy. Re-encoding each state
as a dyadic fraction prevents the activation evaluator's intermediate
denominators from becoming the next state encoding. Each preactivation
query has fraction size envelope $O(S+b+1)$ and precision $p=O(b+1)$.
The activation expression therefore has height
$O((b+1)^3(S+b+1))$, hence at most $O((S+b+1)^4)$.
Its cubic work envelope is $O((S+b+1)^{12})$; the quadratic
long-division envelope for rounding is absorbed into this bound.
Summing over $b+2=O(b+1)$ actual updates gives the stated cost bound.
\end{proof}

The Jacobian approximation also lies in $[1,3]$, so the scalar
preconditioner $D=1/2$ passes the rational test
$|1-\widehat H_u'D|\le1/2$ throughout this box. Thus the linear-refinement
certificate has a concrete preconditioner here. The original equilibrium
map can have derivative magnitude two; the half-damped solver supplies
its own contraction certificate.

The work model charges integer arithmetic by conservative schoolbook
binary-arithmetic envelopes, including long division for dyadic rounding.
The complete update instantiates the same efficient-oracle interface used
by the modular results, with derived work and output-size bounds.
The formal results verify the envelope bounds and the decoded numerical
run; they do not verify native library timing or a low-level bit-machine
implementation. This concrete inference family does not discharge
multidimensional gate realization or the region certificates required
by the adaptive training theorem.

\section{The training mechanism from a bird's-eye view}
\label{sec:bird}
At the start of a pass, freeze the displacement
$\Delta=y^*-F(\theta_0)$ and follow the straight output path
$F(\theta_0)+s\Delta$. The pass needs a lift of this one displacement;
it need not control every output direction equally.

For training, augment the base recurrent weight by programmable bilinear
rank-one channels:
\begin{equation}
 W_{\mathrm{eff}}=W_{\mathrm{base}}+
       \sum_{j=1}^R a_jb_jp_jq_j^T.
\label{eq:modified}
\end{equation}
Both latch coordinates start at zero. Vectors of an unused channel may be
programmed before activation while its recurrent contribution remains exactly
zero. After a channel is activated, it is part of the fixed recurrent operator.
Every subsequent equilibrium evaluation and inference certificate therefore
uses the updated residual and state Jacobian, including all activated terms.

\paragraph{Detect the relevant weakness.}
At a point with training Jacobian $J$, solve
\[
 (JJ^*+\tau^2I)w=\Delta,\qquad v=J^*w.
\]
If $\norm{\Delta-Jv}$ is too large, the normalized solve iterate gives a
weak covector with a substantial component of the displacement. Spectral
estimation is unnecessary. With approximate solves, explicit residual
enclosures and error slack are needed to certify this diagnosis.

\paragraph{Program a residual-aligned repair.}
For an unused channel, its output derivative is a linear function of its
write vector once the read vector is fixed. The trainer finds a bounded
program approximating the normalized pass displacement through this gate.
It activates $(a_j,b_j)=(\rho,0)$, preserving the model output. Witness
pairing is useful for diagnostics or candidate selection, but cannot
replace the realization test. These channels use rank-one adapter
directions related to LoRA~\cite{hu2022}, with a separately programmed
activation policy.
Figure~\ref{fig:repair} summarizes the output-preserving activation and
the certificates needed before executing a repaired pass.

\begin{figure}[H]
\centering
\resizebox{\linewidth}{!}{%
\begin{tikzpicture}[>=stealth,
  box/.style={draw,rounded corners=1pt,minimum height=2em,minimum width=2.6em,
              fill=white,font=\small},
  op/.style={draw,circle,inner sep=1pt,minimum size=1.5em,fill=white,
             font=\small},
  every node/.style={font=\small}]

\node[op]  (sum) at (0,0)   {$+$};
\node[box] (act) at (1.9,0) {$\varphi$};
\node[box,minimum width=4.0em,thick] (x) at (4.3,0) {$x\in\R^{n}$};
\node (u) at (-4.0,0) {$u$};
\node[box] (U) at (-2.8,0) {$U$};
\draw[->] (u)   -- (U);
\draw[->] (U)   -- (sum);
\draw[->] (sum) -- (act);
\draw[->] (act) -- (x);
\node[box] (w)    at (6.5,0) {$w^{T}$};
\node      (yhat) at (8.2,0) {$\widehat y(\theta)$};
\draw[->] (x) -- (w);
\draw[->] (w) -- (yhat);

\node[box] (W0) at (1.9,1.6) {$W_0$};
\draw[->] (x.north) |- (W0);
\draw[->] (W0) -| (sum.north);

\begin{scope}[blue!55!black]
\node[box,draw=blue!55!black] (Nt)  at (3.4,-1.4) {$N^{T}$};
\node[op,draw=blue!55!black]  (eta) at (2.2,-1.4) {$\eta$};
\node[box,draw=blue!55!black] (Cb)  at (1.1,-1.4) {$C$};
\node[box,draw=blue!55!black] (M)   at (0,-1.4)   {$M$};
\draw[->] (x.south) |- (Nt);
\draw[->] (Nt)  -- (eta);
\draw[->] (eta) -- (Cb);
\draw[->] (Cb)  -- (M);
\draw[->] (M)   -- (sum.south);
\end{scope}

\begin{scope}[red!65!black]
\node[box,densely dashed,draw=red!65!black]
      at ($(3.4,-3.0)+(0.09,0.09)$) {\phantom{$q_j^{T}$}};
\node[box,densely dashed,draw=red!65!black] (q) at (3.4,-3.0) {$q_j^{T}$};
\node[op,densely dashed,draw=red!65!black] (gate) at (1.7,-3.0) {$\times$};
\node[box,densely dashed,draw=red!65!black]
      at ($(-0.9,-3.0)+(0.09,0.09)$) {\phantom{$p_j$}};
\node[box,densely dashed,draw=red!65!black,thick] (p) at (-0.9,-3.0) {$p_j$};
\node (ab) at (1.7,-2.3) {$a_jb_j$};
\draw[->,densely dashed] (ab) -- (gate);
\draw[->,densely dashed] ($(x.south)+(0.25,0)$) |- (q);
\draw[->,densely dashed] (q)    -- (gate);
\draw[->,densely dashed] (gate) -- (p);
\draw[->,densely dashed] (p) -- (-0.9,-0.55) -- (sum);
\node[align=left,anchor=west] at (4.9,-3.0)
 {programmable dormant reservoir, $j=1,\dots,R$:\\
  $a_j=b_j=0$ at initialization (exactly inactive);\\
  \textbf{$p_j$ is set at latch time} to align the\\
  gate column with the pass residual\\
  (bounded realization certificate required)};
\end{scope}

\begin{scope}[shift={(-3.9,-4.4)}]
\draw[-] (0,0) -- (0.7,0);
\node[anchor=west] at (0.8,0) {base equilibrium network};
\draw[-,blue!55!black] (0,-0.55) -- (0.7,-0.55);
\node[anchor=west] at (0.8,-0.55) {factorized global channel};
\draw[-,densely dashed,red!65!black] (0,-1.10) -- (0.7,-1.10);
\node[anchor=west] at (0.8,-1.10) {programmable reservoir};
\end{scope}
\end{tikzpicture}}
\caption{The proposed modification uses the recurrent weight in
\eqref{eq:modified}.  Red, dashed: the programmable dormant
reservoir.  While $a_jb_j=0$ the channel is exactly invisible, so its
write vector $p_j$ is not yet data and may legitimately be chosen at
activation.  A latch $(a_j,b_j)\colon(0,0)\to(\rho_*,0)$ is an exact
zero-output operation installing the Jacobian column at activation
$\rho_*c_j$; programming $p_j$ aligns $c_j$ with the normalized pass
residual. Column drift is controlled on the certified pass region.  At most one latch per homotopy pass is ever needed
under the realized-pass hypotheses.}
\label{fig:repair}
\end{figure}
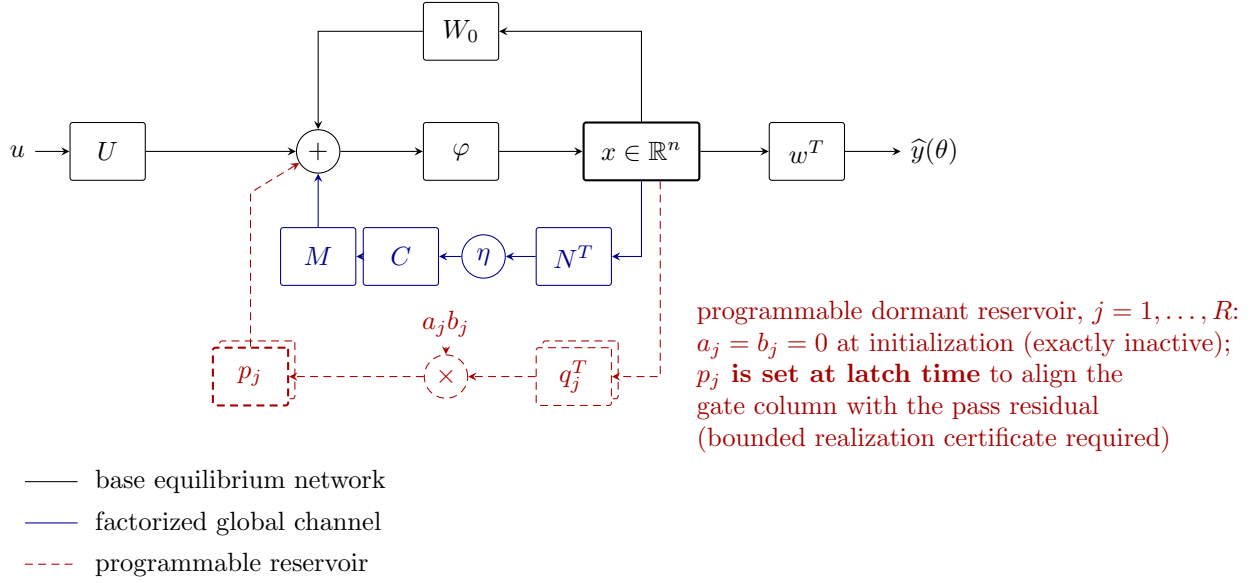

\paragraph{Secure one region, not an eternal column.}
The repair's effective column is allowed to move. Its initial realization
error and its variation over a certified pass ball must fit below the
loaded threshold. These bounds secure the fixed displacement on that
ball. A later pass uses a new displacement and, if necessary, another
unused channel. No permanent global column identity is asserted.

\paragraph{Execute a finite pass.}
Use $N$ steps with step size $1/N$ and inexact lifted velocities. A speed
bound keeps the iterates in the ball. A derivative-Lipschitz bound controls
the accumulated truncation error, while the solve and rounding error have
their own budget. The resulting pass contracts the output residual by
$3/4$. Iteration of this operation gives the precision and reserve count.

\paragraph{Separate success from failed certification.}
The certified trainer accepts only passes whose region, realization and
backend contracts hold. Failure of these checks rejects the pass and its
claimed guarantee. A weak loaded covector certifies a current linearized
obstruction; reserve exhaustion or a failed region check does not prove
global non-realizability of the labels.

\section{Training theorems}
Let $E$ be the Euclidean parameter space and $Y=\R^m$. Adjoints are with
respect to their fixed inner products. All norms below are Euclidean or
induced operator norms.

\begin{definition}[Loaded invariant]\label{def:invariant}
For $\Delta\in Y$ and $\tau>0$, the Jacobian $J$ has the loaded invariant if
\[
 \norm u=1,\quad\norm{J^*u}\le\tau
 \quad\Longrightarrow\quad4\ip{\Delta}{u}^2<\norm\Delta^2.
\]
\end{definition}

\begin{theorem}[Tikhonov dichotomy]\label{thm:tikhonov}
For the exact regularized normal solve, the velocity satisfies
$\tau\norm v\le\norm\Delta$. Either
$2\norm{\Delta-Jv}\le\norm\Delta$, or $w\ne0$ and
$u=w/\norm w$ satisfies
\[
 \norm u=1,\quad\norm{J^*u}\le\tau,
 \quad2\ip\Delta u>\norm\Delta.
\]
In particular the loaded invariant guarantees a lift residual at most
$\norm\Delta/2$.
\end{theorem}
\begin{proof}
Let $A=JJ^*$ and solve $(A+\tau^2I)w=\Delta$. The operator is invertible:
\[
 \ip{(A+\tau^2I)z}{z}=\norm{J^*z}^2+\tau^2\norm z^2,
\]
so its kernel is zero, and finite square dimension gives surjectivity.
Put $v=J^*w$. The equation gives
\begin{equation}
 \Delta-Jv=\tau^2w,\qquad
 \ip\Delta w=\norm v^2+\tau^2\norm w^2.
 \label{eq:energy}
\end{equation}
Cauchy--Schwarz implies
$\tau^2\norm w\le\norm\Delta$ and
$\norm v^2\le\norm\Delta\norm w$, hence
$\tau\norm v\le\norm\Delta$. These claims are immediate if $w=0$.

For the stronger estimate, expand
\[
 \norm\Delta^2=\norm{Aw}^2+2\tau^2\norm v^2+
                         \tau^4\norm w^2.
\]
Since $\norm v^2=\ip{Aw}{w}\le\norm{Aw}\norm w$ and
$(\norm{Aw}-\tau^2\norm w)^2\ge0$, the expansion implies
$\norm\Delta^2\ge4\tau^2\norm v^2$. Therefore
\begin{equation}
 2\tau\norm v\le\norm\Delta.\label{eq:key}
\end{equation}

Suppose $2\norm{\Delta-Jv}>\norm\Delta$. By \eqref{eq:energy},
$2\tau^2\norm w>\norm\Delta$, so $w\ne0$. Normalize
$u=w/\norm w$. Equation~\eqref{eq:key} gives
\[
 \norm{J^*u}=\norm v/\norm w<\tau,
\]
and the energy identity gives
\[
 \ip\Delta u=\norm v^2/\norm w+\tau^2\norm w
           >\norm\Delta/2.
\]
This is the advertised witness. Squaring its positive loading violates
the loaded invariant. No eigendecomposition or smallest-singular-value
estimate is used.
\end{proof}
The solve is a regularized least-squares lift, in the tradition of
Tikhonov and Levenberg--Marquardt methods~\cite{tikhonov1963,levenberg1944,marquardt1963}.

\begin{proposition}[Certified normal-solve error]\label{prop:solve-error}
If $w$ is the exact normal-equation solution and
$\norm{(JJ^*+\tau^2I)\widehat w-\Delta}\le\xi$, then
\[
 \norm{\widehat w-w}\le\xi/\tau^2,
 \qquad\norm{J^*\widehat w-J^*w}\le\norm J\,\xi/\tau^2.
\]
The residual bound must concern the true operator and displacement.
\end{proposition}
\begin{proof}
For an exact solution $w$ and approximate $\widehat w$, let
$z=\widehat w-w$ and $\mathcal N=JJ^*+\tau^2I$. If
$\norm{\mathcal N\widehat w-\Delta}\le\xi$, then
\[
 \tau^2\norm z^2\le\ip{\mathcal Nz}{z}
       \le\xi\norm z,
 \qquad \norm z\le\xi/\tau^2.
\]
Consequently
$\norm{J^*\widehat w-J^*w}\le\norm J\,\xi/\tau^2$.
Errors in $J$, its products, and $\Delta$ must first be converted into a
bound on the true normal-equation residual. Parameter rounding is then
added to the resulting velocity bound.

The exact witness theorem must not be applied without slack to an
approximate solve. A direct certified witness check verifies a unit-vector
enclosure, the upper bound on $\norm{J^*u}$, and the lower bound on its
loading. Alternatively, write
$\Delta'=\mathcal N\widehat w$ and apply the exact theorem to $\Delta'$,
then transfer loading by
$|\ip{\Delta-\Delta'}u|\le\xi$.
\end{proof}

\begin{proposition}[Programming and activation]\label{prop:gate}
Suppose the channel acts through $H((ab)p)$ with $H$ differentiable at
zero and derivative $G_0$. Programming $p$ while $b=0$ changes no output.
After $a=\rho$, the $b$-derivative is $\rho G_0p$. Write
$G=\rho G_0$ for the full post-latch gate.
If a numerical gate $\widehat G$ and a program $p$ satisfy
\[
 \norm{\rho d-\widehat Gp}\le\eta_g,
 \quad\norm{G-\widehat G}\le\gamma,
 \quad\norm p\le S,
\]
then $\norm{\rho d-Gp}\le\eta_g+\gamma S$.
\end{proposition}
\begin{proof}
While $b=0$, $H((ab)p)=H(0)$ for every $a,p$. Both programming and
activation preserve the output exactly as algebraic operations. After
activation, differentiation of $b\mapsto H((\rho b)p)$ at zero gives
$DH(0)[\rho p]=\rho G_0p=Gp$. The gate approximation bound follows from
\[
 \norm{\rho d-Gp}\le\norm{\rho d-\widehat Gp}+
                       \norm{(G-\widehat G)p}
                    \le\eta_g+\gamma S.
\]

For the recurrent adapter, fix its read vector $q$ and let $x_i$ be the
selected equilibrium for sample $i$. With invertible state Jacobian
$A_i=I-D_iW_{\mathrm{eff}}$, its unscaled write-side gate has rows
\begin{equation}
 (G_qp)_i=(q^Tx_i)\,w^TA_i^{-1}D_ip.
 \label{eq:gate-row}
\end{equation}
This formula follows by differentiating the equilibrium equation with
respect to a recurrent perturbation $pq^T$. It is linear in $p$ but depends
on the states, readout and local inverse. The latent inference certificate
controls the inverse in this formula; it does not make $G_q$ surjective.

The adjoint direction $p=G_q^*u/\norm{G_q^*u}$ maximizes pairing with $u$
on the unit write ball, when the adjoint image is nonzero. Its output
column is $G_qG_q^*u/\norm{G_q^*u}$, rather than $u$ in general.
A residual-aligned repair therefore requires solving the bounded
realization problem and checking its residual.
\end{proof}
One may equivalently solve for
the unscaled column and multiply all realization errors by $\rho$.

\begin{theorem}[A stable approximate repair]\label{thm:stable}
Let $\Delta=cd$, $c>0$, $\norm d=1$. At the post-latch start $\theta_0$,
suppose there is a parameter direction $e$ with $\norm e\le1$ and
\[
 \norm{\rho d-J(\theta_0)e}\le\varepsilon.
\]
On a certified region assume $\norm{J(\theta)-J(\theta_0)}\le\omega$.
If
\begin{equation}
 2(\tau+\varepsilon+\omega)<\rho,\label{eq:stable}
\end{equation}
the loaded invariant holds at every point of the region. Each point admits
a lift with residual at most $c/2$ and speed at most $c/\tau$.
\end{theorem}
\begin{proof}
For every point in the certified region,
\begin{align*}
 \norm{\rho d-J(\theta)e}
 &\le\norm{\rho d-J(\theta_0)e}+
          \norm{(J(\theta)-J(\theta_0))e}\\
 &\le\varepsilon+\omega.
\end{align*}
For a unit covector $u$, split its pairing with $\rho d$ into the actual
column and its error. Adjointness gives
\[
 \rho|\ip d u|\le
       |\ip{J(\theta)e}u|+\varepsilon+\omega
       \le\norm{J(\theta)^*u}+\varepsilon+\omega.
\]
Thus any weak unit $u$ satisfies
$\rho|\ip d u|\le\tau+\varepsilon+\omega<\rho/2$.
Since $\Delta=cd$ and $\norm\Delta=c$,
$4\ip\Delta u^2<\norm\Delta^2$. Theorem~\ref{thm:tikhonov} gives the
pointwise lift and velocity bounds. Protection is established only over
the certified region of this pass. No claim that the column remains fixed
after later training is needed.
\end{proof}
An available Lipschitz bound $K$ on a ball of radius $a$ supplies
$\omega\le Ka$. This smallness condition is part of the promise, and can
fail even when the initial programming residual is small.

\begin{theorem}[Finite numerical pass]\label{thm:pass}
Let $F$ have derivative $J$, $K$-Lipschitz on the closed ball
$B=\overline B(\theta_0,a)$. Put $c=\norm{y^*-F(\theta_0)}$.
Suppose the actual step velocities $\widehat v_i$ satisfy on $B$
\[
 \norm{\widehat v_i}\le W\le a,\qquad
 \norm{J\widehat v_i-\Delta}\le c/2+\eta.
\]
If $N\ge1$ and
\begin{equation}
 KW^2/N+\eta\le c/4,\label{eq:numerical}
\end{equation}
the $N$ updates $\theta_{i+1}=\theta_i+\widehat v_i/N$ remain in $B$
and their endpoint has residual at most $3c/4$. For an exact lift $v$ and
$\norm{\widehat v-v}\le\zeta$, $\norm J\le M$, sufficient velocity
contracts are
\[
 W=c/\tau+\zeta,\qquad\eta=M\zeta.
\]
Rounding of a whole update is included by adding $N$ times its parameter
error to the velocity error.
\end{theorem}
\begin{proof}
Let $h=1/N$. Assuming the first $i$ steps are in the ball, their speed
bound gives
\[
 \norm{\theta_i-\theta_0}\le ihW\le W\le a.
\]
Induction proves domain safety for every iterate. Each update segment is
in the ball by convexity.

The derivative-Lipschitz remainder estimate is
\[
 \norm{F(x+hv)-F(x)-hJ(x)v}\le K(hW)^2.
\]
This uses a conservative constant; the usual $K/2$ estimate would improve
the budget but is unnecessary. Adding the lift defect gives
\[
 \norm{F(\theta_{i+1})-F(\theta_i)-h\Delta}
          \le K(hW)^2+h(c/2+\eta).
\]
Sum the $N$ inequalities and telescope. Since $Nh=1$ and
$F(\theta_0)+\Delta=y^*$,
\[
 \norm{F(\theta_N)-y^*}\le KW^2/N+c/2+\eta\le3c/4.
\]

If $v$ is an exact lift and $\norm{\widehat v-v}\le\zeta$, then
\[
 \norm{\widehat v}\le c/\tau+\zeta,
 \quad\norm{J\widehat v-\Delta}
       \le\norm{Jv-\Delta}+\norm J\zeta\le c/2+M\zeta.
\]
For an implemented update with error $r$ relative to
$x+v/N$, the effective velocity is $v+Nr$. This explains the rounding
contribution stated in the theorem.
\end{proof}

\begin{corollary}[One realized repair supplies a finite pass]\label{cor:realized}
Suppose Theorem~\ref{thm:stable} holds on the ball, $J$ is $K$-Lipschitz
there, $c/\tau\le a$, and
$K(c/\tau)^2/N\le c/4$ for an integer $N\ge1$. Then pointwise lifts can
be selected and integrated by $N$ finite steps to reach residual at most
$3c/4$ while remaining in the ball. No Lipschitz assumption on the selected
lifting field is required for this discrete result.
\end{corollary}
\begin{proof}
At each point of the ball Theorem~\ref{thm:stable} supplies a lift with
residual at most $c/2$ and speed at most $c/\tau$. Select one at each
point, defining it arbitrarily outside the ball. Apply
Theorem~\ref{thm:pass} with $W=c/\tau$ and $\eta=0$. Its domain induction
and remainder estimate require only the stated pointwise bounds, rather
than regularity of the selected lift. This proves a real-valued finite
pass construction. For rational execution, a certified approximate backend
must additionally satisfy the inexact contracts; classical pointwise
selection alone does not give a bit-efficient implementation.
\end{proof}

\subsection{A concrete pass schedule}
\label{sec:pass-schedule}
The region and numerical budgets can be stated as scalar acceptance tests.
For a nonzero pass residual $c$, choose a ball radius $a=2c/\tau$ and an
initial repair error $\varepsilon$. A sufficient stability test is
\[
 2\bigl(\tau+\varepsilon+2Kc/\tau\bigr)<\rho.
\]
If the backend's total velocity error satisfies
$\zeta\le c/\tau$ and $M\zeta\le c/8$, choose
\[
 N=\max\{1,\lceil32Kc/\tau^2\rceil\}.
\]
Then $W=2c/\tau$ and the numerical budget of
Theorem~\ref{thm:pass} holds. This schedule needs polynomially many inner
steps when $Kc/\tau^2$ is polynomially bounded. It is conservative and
provides an explicit certified implementation target; the archived
adaptive Gauss--Newton loop uses a different schedule.
\begin{proof}[Verification of the schedule]
The ball gives $\omega\le Ka=2Kc/\tau$. The velocity error bound gives
$\norm{\widehat v}\le c/\tau+\zeta\le2c/\tau=a$, so discrete domain
safety applies. Finally,
\[
 \frac{K(2c/\tau)^2}{N}
   =\frac{32Kc/\tau^2}{N}\frac c8\le\frac c8.
\]
Adding $M\zeta\le c/8$ proves the quarter-residual numerical budget.
The norm bound $M$ and the preactivation and state enclosures must be valid
on the chosen ball, and programming must be included in the post-latch
region certificate.
\end{proof}

\begin{assumption}[Certified pass and backend class]\label{ass:class}
An accepted encoded instance supplies at most $3(b+\ell)$ passes with:
\begin{enumerate}[label=(\roman*),leftmargin=2em]
\item sound DEQ well-posedness, evaluation and derivative certificates on
the entire pass region, and $C^{1,1}$ bounds there;
\item either the loaded invariant on that region or one unused programmable
channel whose post-latch column satisfies Theorem~\ref{thm:stable};
\item actual finite updates satisfying Theorem~\ref{thm:pass}, with their
decoded endpoint identified with the backend output;
\item an encoding of size at most $S$ per backend query and precision at
most $Q$, where $S,Q$ are polynomial in instance length and $b$;
\item a selected complete pass backend, including programming, certificate
checking, inference, linear algebra and rounding, with bit cost at most
$A(S+Q+1)^d$ for uniform family constants $A,d$.
\end{enumerate}
Region bounds may be supplied as validated analytic bounds or enclosures;
testing a finite list of successful solves does not establish them.
The assumptions specify a reduction to an efficient certified backend,
not a universal verifier for arbitrary smooth maps.
\end{assumption}

\begin{theorem}[Certified training with budgeted termination]\label{thm:main}
For an instance in Assumption~\ref{ass:class} with initial residual at most
$2^\ell$, at most $3(b+\ell)$ completed passes produce the encoded output
\eqref{eq:objective}. The number of consumed channels is at most
$3(b+\ell)$, and it suffices to provision that many channels.
The total bit cost of the accepted encoded run is at most
\[
 3(b+\ell)A(S+Q+1)^d.
\]
Thus the trainer is polynomial in instance length and requested bits when
the displayed budgets, including $\ell$, are polynomially bounded.
If pass $k$ has residual $c_k$, travel at most $c_k/\tau$, and each latch
costs at most $\rho$ parameter travel, then the total travel is at most
\[
 4c_0/\tau+R_{\mathrm{used}}\rho.
\]
Programmed-vector changes and rounding add their separately certified
travel budgets. A global domain ball is sufficient only if it contains
this full travel budget and the pass regions used for certification.

More generally, execute the trainer through
Theorem~\ref{thm:budgeted-certification}, with successful-answer predicate
$\norm{F(\operatorname{decode}(y))-y^*}\le2^{-b}$, sound acceptance,
and the stated all-input transition-cost bound. Every run then terminates
with certified success or inconclusive within the polynomial wrapper
budget. For a fixed backend family satisfying Assumption~\ref{ass:class},
choose the budget to cover its complete encoded trace and dispatch; the
wrapper preserves the above success, pass, and channel guarantees.
Failure of a certificate or exhaustion of the budget is inconclusive,
not a certificate of superpolynomial training complexity.
\end{theorem}
\begin{proof}
Pass activation preserves the output, so the residual used in the pass
is also the residual of the preceding encoded endpoint. By the finite
pass theorem, $c_{k+1}\le(3/4)c_k$ whenever a nontrivial accepted pass is
executed. An already successful endpoint can be retained with no further
latches. Inductively,
\[
 c_k\le(3/4)^kc_0.
\]
Because $(3/4)^3=27/64\le1/2$,
\[
 c_{3(b+\ell)}\le2^{-(b+\ell)}2^\ell=2^{-b}.
\]
This argument requires contracts for only the finite run, not an infinite
extension of it. Each pass consumes at most one distinct unused channel,
so summing the latch count gives at most $3(b+\ell)$.

The encoded backend run is defined recursively by applying the selected
pass backend to its preceding encoded output at the prescribed precision.
The pass certificate identifies that output's decoded parameter with the
finite numerical endpoint, tying convergence to the same calls whose work
is counted. For every call, its complete bit cost is bounded by
$A(S+Q+1)^d$. Summing over $3(b+\ell)$ calls gives
$3(b+\ell)A(S+Q+1)^d$. The degree and coefficient are uniform for the
chosen algorithm family. If the budgets are polynomially bounded in
instance length and $b$, their composition is polynomial. Real-arithmetic
operation counts alone would not justify this conclusion.

For the travel statement,
\[
 \sum_{k<K}c_k\le c_0\sum_{k<K}(3/4)^k\le4c_0.
\]
Adding unit-time pass travel $c_k/\tau$ and the latch movements gives the
displayed bound. If each approximate velocity additionally has a
$\zeta_k$ error, unit-time pass travel additionally costs at most
$\sum_k\zeta_k$. Programming vectors in the parameter metric similarly
needs its own bound, even though their dormant output contribution is
zero. Pointwise domain safety of an accepted pass does not automatically
prove coverage of its entire certification ball by a previously chosen
global domain ball.
Finally apply Theorem~\ref{thm:budgeted-certification} with the training
answer predicate. Pass certificates justify acceptance on the original
successful trace; the wrapper preserves it when the clock is sufficient.
The all-input cost requirement also covers unsuccessful searches and
checks, which Assumption~\ref{ass:class} alone does not bound.
\end{proof}

\begin{remark}[Reserve width and computational size]
The guarantee counts channels, not stored scalar entries. Its bound is
$O(b+\ell)$, independent of output dimension \emph{as a count}, while
gate realization and matrix-free products still depend on $m$ and $n$.
Programmable rank-one channels store $O(nR)$ vector entries. The theorem
does not guarantee that a rank-one gate can realize arbitrary batch
directions. For a fixed $q$, its image dimension is at most the write
dimension; this restriction is not removed by increasing latch amplitude.
\end{remark}

\subsection{Practical implementation: matrix-free loaded training}
\label{sec:matrix-free}
The regularized solve in the trainer can be implemented without assembling
the batch training Jacobian or forming a matrix inverse. Its primitive
operations are a Jacobian-vector product $v\mapsto Jv$ and an adjoint
product $q\mapsto J^*q$. The operator passed to an iterative solver is
\[
 \mathcal Nq=J(J^*q)+\tau^2q,
 \qquad \mathcal N\widehat w\simeq\Delta,
 \qquad \widehat v\simeq J^*\widehat w.
\]
One application of $\mathcal N$ uses one product of each kind and vector
operations. This replaces storage of an $m\times p$ Jacobian and an
$m\times m$ normal matrix by product routines and solver workspace.
The state solves, recurrence factors, batch states, and programmed channels
still have their own storage costs.

\paragraph{Products through an equilibrium.}
For batch sample $i$, cache its accepted equilibrium $x_i$ and activation
derivative $D_i$, and set $E_i=I-D_iW_{\mathrm{eff}}$.
A parameter direction with induced perturbations
$\delta W_{\mathrm{eff}},\delta U,\delta c,\delta w,\delta c_o$ gives
\[
 E_i\delta x_i
   =D_i(\delta W_{\mathrm{eff}}x_i+\delta Uu_i+\delta c),
 \qquad (Jv)_i=x_i^T\delta w+\delta c_o+w^T\delta x_i.
\]
For an output covector $q$, the corresponding adjoint state solve is
$E_i^T\lambda_i=q_iw$. Contracting $D_i\lambda_i$ with the parameter
directions gives the recurrent and input contributions to $J^*q$;
the readout contributions are $\sum_iq_ix_i$ and $\sum_iq_i$.
These contractions use the structured weight parameterization, including
the active bilinear channels, rather than a dense recurrent gradient.
The factorized Jacobian reduction of Theorem~\ref{thm:reduction} supplies
state and adjoint solves through factorizations of the local blocks and
the reduced matrix; no explicit inverse need be stored. Active channels
must be included in those factors and conditioning certificates.
The state Jacobian can be nonsymmetric, so conjugate gradients is used
for the regularized normal system, not automatically for the state solve.

\paragraph{Iterative solution and conditioning.}
For $\tau>0$, $\mathcal N$ is symmetric positive definite even when $J$
is rank deficient. If $\|J\|\le M$, its spectral condition number satisfies
\[
 \kappa(\mathcal N)\le1+M^2/\tau^2.
\]
Conjugate gradients (CG) therefore provides a standard matrix-free solver.
In exact arithmetic, its classical estimate gives a sufficient iteration
count of order
$O(\sqrt\kappa\log(2+\kappa\|\Delta\|/\xi))$ for normal residual tolerance
$\xi$, with $\kappa=\kappa(\mathcal N)$~\cite{barrett1994templates}.
A symmetric positive definite preconditioner may reduce this count;
constructing and applying it are part of the backend cost.
Regularization prevents singularity, but small $\tau$ can still cause
expensive solves. Polynomial cost requires controlled conditioning,
product precision, and encoded arithmetic, as in Assumption~\ref{ass:class}.

\paragraph{Acceptance with inexact products.}
The stopping test must bound the residual of the true normal equation.
Suppose a checked product routine returns $\widehat a$ with
$\|\widehat a-\mathcal N\widehat w\|\le\varepsilon_{\mathrm{act}}$,
and $\|\widehat\Delta-\Delta\|\le\varepsilon_\Delta$.
Including arithmetic error in these enclosures gives the acceptance bound
\[
 \|\mathcal N\widehat w-\Delta\|
 \le\|\widehat a-\widehat\Delta\|
       +\varepsilon_{\mathrm{act}}+\varepsilon_\Delta\le\xi.
\]
Proposition~\ref{prop:solve-error} then converts $\xi$ into solution and
velocity error. If the final adjoint product has error at most
$\varepsilon_v$, the computed velocity satisfies
\[
 \|\widehat v-J^*w\|\le M\xi/\tau^2+\varepsilon_v.
\]
This quantity and whole-update rounding must fit the finite-pass budget
of Theorem~\ref{thm:pass}. A solver's recursively updated residual alone
does not account for inaccurate equilibrium states, derivative products,
or rounding. Those errors need validated bounds, and approximate witnesses
need the slack checks described after Proposition~\ref{prop:solve-error}.

\paragraph{A practical execution policy.}
At each inner update, compute certified states, apply the two product
routines in the regularized solver, check its true-residual enclosure,
and form the accepted velocity. Warm starts and reuse of factorizations
at the same parameter point can reduce work. A parameter change requires
updated products or a certified bound for reusing them.
Trial points must satisfy the inference and pass-region tests, and an
unused channel is programmed only after its residual-aligned realization
and stability checks succeed. Line searches and adaptive damping are
useful practical choices, but their accepted updates must satisfy the
finite-pass contract to retain the theorem's guarantee. The loaded trigger
requires no Lanczos margin monitoring or spectral decomposition.

The experiments in Section~\ref{sec:numerics} already use matrix-free
products and iterative normal solves; their product counts illustrate this
implementation route.
This discussion describes
a practical candidate backend within the stated certification contracts.

\section{An exponential obstruction to the global finite-pass certificate}
\label{sec:exponential-pass}
The polynomial backend hypothesis is substantive even when inference is
well conditioned and a repair is successfully realized. We construct a
three-state $\tanh$ DEQ for which a pass certified by the single global
curvature and speed budget of Theorem~\ref{thm:pass} requires exponentially
many updates in the binary input length. The obstruction concerns that
certificate and a specified trainable coordinate, not all algorithms for
the training problem. In fact, a directly certified update solves the
same interpolation task efficiently.

\subsection{The training instance and successful repair}
For an integer $L\ge5$, put $\kappa=2^{-L}$ and consider
\[
 x=\tanh\left(a\theta pq^Tx+\begin{pmatrix}1\\0\\0\end{pmatrix}\right),
 \qquad p=\begin{pmatrix}0\\1\\\kappa^{-1}\end{pmatrix},
 \quad q=\begin{pmatrix}1\\0\\0\end{pmatrix},
 \qquad F=x_2+\frac{\kappa}{32}x_3.
\]
There is one training sample, with target $c=1/128$. All supplied
coefficients are rational and their total binary encoding length is
$\Theta(L)$. The large entry $\kappa^{-1}=2^L$ requires only $L+1$
binary digits. Initially $a=\theta=0$. Program $p,q$ while the channel
is dormant, then activate $a=\kappa$ without changing the output.
The backend subsequently fixes $a,p,q$, the readout and all other
coefficients, and trains only $\theta$. All parameter balls and derivatives
in this section refer to this one-dimensional trainable slice.

Write $t=\tanh1$, so $3/4<t<4/5$. These inequalities follow from
$7<e^2<9$. After activation, the unique equilibrium and output are
\[
 x_1=t,\qquad x_2=\tanh(\kappa t\theta),\qquad
 x_3=\tanh(t\theta),\qquad
 F_\kappa(\theta)=\tanh(\kappa t\theta)
                  +\frac{\kappa}{32}\tanh(t\theta).
\]
The initial residual is $c$. For example, the fixed accuracy request
$b=8$ asks for residual at most $c/2$. Thus the accuracy request need
not grow with $L$ in the lower bound below.

\begin{proposition}[Conditioned inference and stable repair]
\label{prop:exp-repair}
Inference has a unique equilibrium for every $\theta$. At that equilibrium,
$\|H_x^{-1}\|<3$. On $|\theta|\le R=1/(8\kappa)$, the stable-repair
theorem holds with
\[
 \rho=\kappa,\qquad \tau=\kappa/8,\qquad
 \varepsilon=\kappa/4,\qquad \omega=3\kappa/64.
\]
A derivative-Lipschitz bound on this ball is $K=\kappa/8$.
\end{proposition}
\begin{proof}
The equilibrium equations are triangular. The state Jacobian is $I-A$,
where only $A_{21}=\kappa\theta\operatorname{sech}^2(\kappa t\theta)$
and $A_{31}=\theta\operatorname{sech}^2(t\theta)$ can be nonzero.
Since $A^2=0$, its inverse is $I+A$. The inequality
$|z|\operatorname{sech}^2z\le1$ gives
$\|H_x^{-1}\|\le1+\sqrt2/t<3$.

Let $J=F_\kappa'$. Before activation, the $\theta$-derivative is zero;
after activation $J(0)=(33/32)t\kappa$. In particular,
$|\kappa-J(0)|\le\kappa/4$. Furthermore,
\[
 0\le J(0)-J(\theta)
 =\kappa t\tanh^2(\kappa t\theta)
   +\frac{\kappa t}{32}\tanh^2(t\theta)
 \le\kappa\left((\kappa\theta)^2+\frac1{32}\right)
 \le\frac{3\kappa}{64}.
\]
Here $|\tanh z|\le\min\{|z|,1\}$ was used. With the unit parameter
direction $e=1$ and output direction $d=1$, the stability test is
$2(\tau+\varepsilon+\omega)=27\kappa/32<\rho$.
Finally, differentiating again yields
$|F_\kappa''|\le2\kappa^2+\kappa/16\le\kappa/8$, since
$\kappa\le1/32$.
\end{proof}

\subsection{An exponential lower bound for the certificate}
\begin{theorem}[Exponential globally certified pass length]
\label{thm:exp-pass}
Let $L\ge10$. Consider a pass from $\theta_0=0$ with fixed displacement
$\Delta=c$, using $N$ updates of size $\widehat v_i/N$. Suppose a
single ball $[-R,R]$, a valid Lipschitz constant $K$ for $J$ on that
ball, a speed bound $W\le R$, and an error budget $\eta\ge0$ certify
the pass through
\[
 |\widehat v_i|\le W,\qquad
 |J(\theta_i)\widehat v_i-c|\le c/2+\eta,\qquad
 KW^2/N+\eta\le c/4.
\]
Then $N\ge2^{L-14}$. Consequently an implementation that executes these
$N$ updates requires $\Omega(2^L)$ operations. The bound applies regardless
of how its lifted velocities or regularization thresholds are chosen.
\end{theorem}
\begin{proof}
The numerical budget gives $\eta\le c/4$. At the first update,
\[
 W\ge\frac{c/2-\eta}{J(0)}\ge\frac{c}{4\kappa},
\]
because $0<J(0)<\kappa$. Hence $R\ge2$ for $L\ge10$, and the
interior of the certification ball contains $\theta_*=1/t<4/3$.
Both terms of $F_\kappa''(\theta_*)$ are negative, so
\[
 K\ge|F_\kappa''(\theta_*)|
 \ge\frac{\kappa}{16}t^3(1-t^2)
 >\frac{243\kappa}{25600}>\frac{\kappa}{128}.
\]
In particular $K,W>0$, and the numerical budget implies $\eta<c/4$.
Using $(c/2-\eta)^2-c(c/4-\eta)=\eta^2\ge0$, we obtain
\[
 N\ge\frac{KW^2}{c/4-\eta}
 \ge\frac{K(c/2-\eta)^2}{J(0)^2(c/4-\eta)}
 \ge\frac{Kc}{J(0)^2}
 \ge\frac1{16384\kappa}=2^{L-14}.
\]
Each explicitly executed update costs at least one operation. Since the
instance length is $\Theta(L)$, this is exponential in that length.
This conclusion does not cover algorithms that bypass the updates or
use a different certificate.
\end{proof}

The bound is not a consequence of deliberately overestimating curvature:
the proof bounds every valid $K$ from below using an actual second
derivative inside the required ball. It also allows inexact velocities.
A concrete admissible schedule is
\[
 K=\kappa/8,\qquad W=R=1/(8\kappa),\qquad
 N=2/\kappa=2^{L+1},\qquad \eta\le c/8.
\]
Indeed, $KW^2/N=c/8$. The exact scalar Tikhonov velocity is
$v=cJ/(J^2+\tau^2)$ with $\tau=\kappa/8$; Proposition~\ref{prop:exp-repair}
and Theorem~\ref{thm:tikhonov} supply its lift bounds. Since
$|J|\le\kappa$, velocity errors at most $2^{-12}$ fit both the displayed
speed and defect budgets. Whole-update errors of order $2^{-L-20}$
contribute only a constant small velocity error after multiplication by
$N$. Stored states have magnitude $O(2^L)$, so a dyadic implementation
can use $O(L)$ bits at this fixed target accuracy. Activation evaluations
at large arguments can use the bound
$1-|\tanh z|\le2e^{-2|z|}$ to certify saturation; other arguments require
only precision growing polynomially with $L$. These observations describe
a finite-precision realization, not a verified low-level implementation.

The sharper variation bound $\omega$ in Proposition~\ref{prop:exp-repair}
is essential. Substitution of the coarse bound $KR$ into the extra
sufficient stability test of Section~\ref{sec:pass-schedule} would reject this schedule:
$KR=1/64$ does not shrink with $\kappa$. The example therefore concerns
the general finite-pass certificate, not an accepted execution of that
more restrictive sufficient test.

\subsection{A fast update and the scope of the obstruction}
\begin{proposition}[Directly certified interpolation]
\label{prop:exp-direct}
For $L\ge5$, the rational update $\theta=c/\kappa$ satisfies
$|F_\kappa(\theta)-c|<c/3$, and hence meets the fixed accuracy request
$b=8$. Its parameter encoding has $O(L)$ bits.
\end{proposition}
\begin{proof}
At this point
$F_\kappa(\theta)=\tanh(tc)+(\kappa/32)\tanh(tc/\kappa)$.
For $z\ge0$, $\tanh z\ge z-z^3/3$, as follows by integrating
$1-\operatorname{sech}^2z=\tanh^2z\le z^2$.
Consequently
\[
 F_\kappa(\theta)>3c/4-c^3/3>2c/3,
 \qquad
 F_\kappa(\theta)<4c/5+\kappa/32\le4c/5+c/8<c.
\]
The stated error bound follows, and $\theta=2^{L-7}$ is rational.
\end{proof}

Thus this family is not an intrinsic hardness result for DEQ training,
nor a reduction from the threshold-network problems of Blum and Rivest
\cite{blum1992}. A global Taylor remainder budget combines a large
parameter displacement with curvature concentrated near the start and
can be exponentially pessimistic. Directional remainder certificates or
local subdivision may avoid that expense. The exponentially large
$\tau^{-1}$ and programmed-vector norm have short encodings; the complete
pass backend just exhibited nevertheless violates the polynomial work
promise in Assumption~\ref{ass:class}. The conditional training theorem
is therefore unaffected.

\section{Deep DEQ and feedforward extensions}
\subsection{A certificate class closed under composition}
Define the inference class through recursively annotated architectures.
Its leaves are certified implicit layers or explicit modules. Serial,
parallel and residual nodes propagate domain, sensitivity and evaluation
certificates; the implicit layers retain their architecture-specific
reduced-margin certificates. Closure means closure of these annotated
presentations with their quantitative budgets.

\begin{proposition}[Serial implicit inference closure]\label{prop:implicit-stack}
For a two-layer equilibrium stack
$H_1(u,x_1)=0$, $H_2(x_1,x_2)=0$, its state Jacobian is
\[
 \mathcal J=\begin{pmatrix}A&0\\C&D\end{pmatrix},\qquad
 \mathcal J^{-1}=\begin{pmatrix}A^{-1}&0\\-D^{-1}CA^{-1}&D^{-1}\end{pmatrix}.
\]
If $\norm{A^{-1}}\le K_1$, $\norm{D^{-1}}\le K_2$ and $\norm C\le P$,
then in the product max norm
\[
 \norm{\mathcal J^{-1}}\le\max\{K_1,K_2(1+PK_1)\}.
\]
Layerwise selected equilibria define a selected equilibrium of the stack;
layerwise uniqueness gives stack uniqueness on the corresponding domain.
\end{proposition}
\begin{proof}
Solve $Av=r_1$ first, then $Dw=r_2-Cv$. This gives the inverse formula
and bounds $\norm v\le K_1\norm r$,
$\norm w\le K_2(1+PK_1)\norm r$ in the product max norm.
Sequential layer selection proves existence of the selected stack state.
For uniqueness, the first layer fixes $x_1$ and the second then fixes $x_2$.
Domain inclusion must hold at each step.
\end{proof}

For an inference chain of Lipschitz constants $L_i$, enclosed local output
errors $\epsilon_i$ propagate as
$e_i\le L_i e_{i-1}+\epsilon_i$. Thus requested local precision must
account for downstream amplification. Parallel nodes use block diagonal
solves and sum their work; residual nodes add the explicit identity map.
A polynomial-size tree has polynomial total inference cost when its
propagated inverse, derivative, domain and precision budgets are polynomial
and its leaf algorithms have polynomial costs with controlled output sizes.
This is a closed certified subclass. A single feedback layer with fixed
interface dimension is not closed under arbitrary stacking: the total
interface dimension can grow,
and a chain of bounded sensitivities can have exponentially large gain.

The input-sweep continuation and finite tracker are treated separately in
Section~\ref{sec:inference-homotopy}. Saddle-coverage and reservoir-refresh
certificates are not needed for this inference closure statement.

\subsection{Training regularity of composed models}
\label{sec:extensions}
A finite stack of equilibrium modules supplies an output map after each
module has its own sound inference certificate. The loaded training
results then apply to the composed map. Polynomial depth alone does not
bound propagated derivative constants; these must be certified.

\begin{proposition}[Composition with trainable modules]\label{prop:composition}
For $h(z,\vartheta)=g(f(z),\vartheta)$,
\[
 Dh[v,w]=Dg[Df\,v,w].
\]
Using the product max norm, joint derivative bounds $L_f,L_g$ yield
$\norm{Dh}\le L_g\max\{L_f,1\}$. If the derivative Lipschitz constants
are $K_f,K_g$ and $f$ is $L_f$-Lipschitz, a sufficient derivative variation
bound is
\[
 K_h\le K_g\max\{L_f,1\}^2+L_gK_f.
\]
The direct downstream parameter contribution is included.
\end{proposition}
\begin{proof}
Let $\mathcal L(z,\vartheta)=(f(z),\vartheta)$. Its derivative is
$D\mathcal L[v,w]=(Df\,v,w)$, whose operator norm in product max norms is
at most $\max\{L_f,1\}$. The chain rule gives $Dh=Dg\,D\mathcal L$ and
the first-order estimate.

For two points, write the derivative difference as
\[
 Dg(\mathcal Lz)D\mathcal L(z)-Dg(\mathcal Lz')D\mathcal L(z')
 =\bigl[Dg(\mathcal Lz)-Dg(\mathcal Lz')\bigr]D\mathcal L(z)
    +Dg(\mathcal Lz')\bigl[D\mathcal L(z)-D\mathcal L(z')\bigr].
\]
The lift map is $\max\{L_f,1\}$-Lipschitz; its derivative variation is
bounded by $K_f$. Taking norms gives
$K_g\max\{L_f,1\}^2+L_gK_f$. In particular, an independent downstream
parameter is passed through with derivative one. A product rule
$L_gL_f$ alone can omit this contribution.
\end{proof}
Equivalent Euclidean product norms introduce explicit norm-equivalence
factors; they must be budgeted before applying the Hilbert-space training
theorems. These compositional bounds concern smoothness, not adapter
realizability or equilibrium uniqueness of the entire stack.

For feedforward networks, no equilibrium existence certificate is needed.
Forward and backward products replace implicit sensitivity solves. The
same programming, column stability and finite-pass hypotheses remain;
placing adapters in an arbitrary layer does not automatically satisfy
them. This extension uses the same abstract training lemmas and the
parameter-aware composition rule rather than a separate training theory.

\section{Numerical evidence and formal verification scope}
\label{sec:numerics}

A Python prototype (included in ancillary files) compares pseudo-arclength continuation with loaded
Tikhonov loops on a compact digits $3$ versus $8$ task ($m=10$, hidden
width $7$, eight base parameters, eight dormant channels, five seeds).
The recorded aggregate values are:
\begin{table}[H]
\centering
\small
\begin{tabular}{lrrrrr}
\toprule
method & success & median RMSE & mean latches & $J$-products & time (s)\\
\midrule
pseudo-arclength & $3/5$ & $3.3\cdot10^{-12}$ & $2.6$ & $13273$ & $8.6$\\
loaded, fixed bank & $5/5$ & $1.6\cdot10^{-13}$ & $4.6$ & $820$ & $0.9$\\
loaded, programmable & $5/5$ & $3.8\cdot10^{-13}$ & $6.4$ & $700$ & $1.1$\\
\bottomrule
\end{tabular}
\caption{Recorded heuristic DEQ comparison on five digits-task seeds.
Success means RMSE below $10^{-8}$ after final polishing.}
\label{tab:deq-comparison}
\end{table}

The prototype maximizes witness pairing, and its equilibrium
flag measures numerical convergence rather than uniqueness. The values
illustrate the potential efficiency of a loaded trigger; they do not
validate the realization and region contracts of Theorem~\ref{thm:main}.

\paragraph{Feedforward comparison.}\label{par:feedforward-comparison}

The prototyped feedforward comparison includes a no-reserve full-training
control that succeeds on all five seeds. Both frozen-body adapter
configurations succeed on one of five seeds and exhaust their reserves.
This demonstrates why a reserve count must be distinguished from a
quantitative gate-realization guarantee.

\paragraph{Selected-branch inference through two folds.}
The deterministic floating-point experiment follows the oriented equilibrium
curve
\[
 H(u,x)=x-\tanh(2x+u)=0
\]
from $x=-4/5$ toward increasing $x$, through both simple folds, to the selected
crossing of $u=-0.51$ beyond the second fold. It uses an adaptive
pseudo-arclength predictor and a bordered Newton corrector. As a control, a
monotone input sweep uses the same residual and warm-started Newton corrections
but does not carry the branch orientation.

\begin{table}[H]
\centering
\small
\begin{tabular}{rrrrrr}
\toprule
$b$ & charts & corrections & selected residual & endpoint error & input-sweep branch error\\
\midrule
8  & 35 & 31 & $4.25\cdot10^{-6}$  & $1.76\cdot10^{-5}$ & $1.77$\\
16 & 35 & 53 & $3.79\cdot10^{-10}$ & $1.57\cdot10^{-9}$ & $1.77$\\
24 & 35 & 69 & $3.79\cdot10^{-10}$ & $1.57\cdot10^{-9}$ & $1.77$\\
32 & 35 & 81 & $<10^{-15}$          & $<10^{-15}$         & $1.77$\\
40 & 35 & 91 & $<10^{-15}$          & $<10^{-15}$         & $1.77$\\
\bottomrule
\end{tabular}
\caption{Floating-point continuation of the selected tanh branch. The chart
count is independent of requested precision in this experiment, while
corrector work grows with precision. The monotone input sweep ends at a
different equilibrium despite an essentially zero terminal residual.}
\label{tab:fold-numerics}
\end{table}

Every pseudo-arclength run crossed two sign changes of $H_x$. Along the
$b=32$ trace, the sampled full-row margin $\sqrt{H_x^2+H_u^2}$ stayed above
$0.448$, the inverse norm of the unit-tangent bordered system stayed below
$2.231$, and $|H_x|$ fell to $0.0292$ near a fold. The selected endpoint was
$x=0.787790669396$, whereas the monotone input sweep returned
$x=-0.986134444248$. Thus residual accuracy alone did not establish the
selected-branch answer predicate. The terminal reference value was computed
independently by 80-digit decimal bisection on a fixed bracket. The tracker
itself uses ordinary double precision, not directed interval arithmetic, so
this is evidence about the mechanism, not a certified execution or evidence
for the general polynomial-complexity theorem.

The quantitative statements above are machine-checked in Lean 4 at the
level stated with each result, from the factorized kernel reduction and
the training passes to the budgeted wrapper, the encoded tracker and
scalar backend, the local chart layer, and the geometric and
algorithmic content of the fold theorem. The correspondence between
each statement and its checked declarations is recorded in source
comments at the statement; the complete development is a single file in
the ancillary material with a pinned toolchain, and its audited
declarations use only the standard logical axioms. Results presented as
paper proofs---the trial search, terminal monitor, and assembly theorem
of Section~\ref{sec:automatic-continuation}, the arclength
reparametrization and rational certificate synthesis of
Section~\ref{sec:fold-continuation}, backend cost summation,
source-specific adapters, and the archived numerical runs---are not
machine-checked.

\subsection{Certified coverage and boundaries at a glance}
\label{sec:coverage-summary}
Table~\ref{tab:coverage} summarizes the accepted problem classes, the
quantitative promises each route consumes, the guarantee it returns, and
its cost regime. Every row
is a conditional certificate: on instances outside its promise class the
budgeted wrapper of Theorem~\ref{thm:budgeted-certification} still
terminates, returning inconclusive rather than an uncertified answer.
The final two rows record proved boundaries, not failures of soundness:
the finite-pass lower bound concerns one specific global certificate,
and the same instance is solved directly by
Proposition~\ref{prop:exp-direct} in $O(L)$ bits.

\begin{table}[t]
\centering
\scriptsize
\begin{tabular}{@{}p{0.17\linewidth}p{0.33\linewidth}p{0.28\linewidth}p{0.14\linewidth}@{}}
\toprule
route & required promises & guarantee & bit cost\\
\midrule
input sweep\newline
(Thms.~\ref{thm:input-sweep}, \ref{thm:rounded-tracker})
& $C^2$ residual, certified start root, boundary exclusion, tube radius
$\rho$ with $\norm{H_x^{-1}}\le K$, $\norm{H_s}\le B$, second-derivative
bound; polynomial encodings
& unique selected branch tracked to $2^{-b}$
& polynomial in $L+b$\\
\addlinespace
default continuation for certified solver classes\newline
(Thm.~\ref{thm:solver-homotopy}, Cor.~\ref{cor:published-coverage})
& instance certified for an established solver: contraction with gap
$\gamma$ in weighted, auxiliary, or logarithmic coordinates; certified
evaluation and decoding
& same equilibrium to $2^{-b}$ by the one continuation mechanism;
explicit coverage of the quantified regimes
of~\cite{elghaoui2021,winston2020,jafarpour2021,sittoni2024,sato2026}
& $O(\gamma^{-3}(p{+}1))$ solver calls; polynomial if $\gamma^{-1}$
polynomial\\
\addlinespace
outlined exceptions: direct solutions\newline
(Sec.~\ref{sec:solver-inheritance})
& separately certified class without a contraction gap, e.g.\ triangular
ReLU with $A_{ii}<1$
& equilibrium by direct rational substitution, outside the continuation
route
& polynomial\\
\addlinespace
fold continuation: new class\newline
(Thm.~\ref{thm:fold-extension}, Ex.~\ref{ex:tanh-fold})
& row-rank margin $\sigma$, curve length $L$, $C^2$ bound, terminal
bracket with $|du/d\ell|\ge\beta$; chart-validation contracts
& selected branch through simple folds to $2^{-b}$, where no inherited
contraction certificate exists on the route; $O(1{+}L/r)$ charts
& polynomial\\
\addlinespace
automatic charts\newline
(Lem.~\ref{lem:auto-box}, Thm.~\ref{thm:auto-assembly})
& certified seed and orientation; certified activation evaluators;
neighborhood, margin, variation, and
length bounds; terminal monitor with margins $\tau,\beta$
& self-generated chart path; selected equilibrium to $2^{-b}$ or
inconclusive
& polynomial given an inverse-polynomial progress scale\\
\addlinespace
block elimination\newline
(Sec.~\ref{sec:block-extension})
& permuted block-triangular tanh DEQ with per-block margins
& equilibrium to $2^{-b}$ even when the full Euclidean inverse bound is
$2^{\Theta(L)}$
& polynomial precision overhead\\
\addlinespace
loaded training\newline
(Thm.~\ref{thm:main})
& gate realization and column stability on pass regions; well-posed
inference; finite-update budgets
& interpolation to $2^{-b}$ in $O(b{+}\ell)$ passes
& polynomial\\
\midrule
global finite-pass certificate\newline
(Thm.~\ref{thm:exp-pass})
& single ball, curvature, and speed budget on the $\kappa=2^{-L}$
three-state instance
& certificate forces $N\ge2^{L-14}$ updates despite conditioned
inference and a successful repair
& exponential for this certificate\\
\addlinespace
chart-count boundary\newline
(Sec.~\ref{sec:automatic-continuation})
& polynomial encoding length without a positive progress scale
& soundness retained; completeness may need exponentially many charts
& possibly exponential\\
\bottomrule
\end{tabular}
\caption{Certified coverage and proved boundaries. Promises are inputs
to the certificates; no row asserts unconditional convergence, and every
route reports certified success or inconclusive within its budget.}
\label{tab:coverage}
\end{table}

\section{Conclusion}
We have developed a certified continuation framework for equilibrium
computation and interpolation training in DEQs under quantitative
certificate and backend promises. The
operational outcome is certified success or inconclusive termination
within a fixed polynomial budget. Supplied old inference certificates
retain zero inference repairs, while sequential block elimination adds
explicit coverage with only polynomial precision overhead despite large
inter-block couplings. This extends the concrete uniform-inverse promise
and uses established contraction ideas. The solver-inheritance theorem
makes one continuation mechanism the polynomial default for explicitly
quantified prior solver classes, while bordered continuation strictly
extends coverage to a new class through simple folds under the
additional length and terminal certificates.
Theorems and lemmas in Section~\ref{sec:automatic-continuation} further
replace a supplied chart list by rational local generation and oriented
assembly, conditional on the seed, quantitative neighborhood bounds, and
terminal-monitor contract. These new geometric results are paper proofs.
Neither result asserts exhaustive coverage of every published convergence
theorem. For inference,
\hyperref[thm:reduction]{Theorem~\ref*{thm:reduction} (reduced denominator)} exposes feedback
singularities, and \hyperref[thm:transfer]{Theorem~\ref*{thm:transfer} (local transfer)}
provides conditioning bounds for the stated block structure.
\hyperref[thm:input-sweep]{Theorem~\ref*{thm:input-sweep} (compact input sweep)} continues a
selected start root beyond the contractive setting. Its hypotheses include
a bounded state region, exclusion of boundary zeros, and invertible state
Jacobians at every zero. For numerical tracking, a certified tube must
supply the uniform bounds
\[
 \|H_x^{-1}\|\le K,\qquad \|H_s\|\le B,\qquad
 \|D^2H\|\le H_2.
\]
Choose a tracking radius $r$ inside that tube with $KH_2r\le1/4$ and
$r\le2^{\ell_r}$, where $\ell_r$ is a nonnegative integer. The
\hyperref[thm:rounded-tracker]{Theorem~\ref*{thm:rounded-tracker} (rounded tracker)} permits the schedule
\[
 N=\max\{1,\lceil2KB/r\rceil\},\qquad m=b+\ell_r+2,\qquad
 \delta\le\min\{r/8,2^{-(b+2)}\}.
\]
The initial state must be within $r/2$ of the start root; each correction's
solve and rounding errors must satisfy $K\eta+\zeta\le\delta$.
All corrections remain in their certified balls, and the computed endpoint
has state error at most $2^{-b}$. There are exactly $Nm$ corrections.
If each complete correction costs at most $A(S+P+1)^d$, where $S$ bounds
query size and $P$ bounds working precision, their total work is at most
\[
 N(b+\ell_r+2)A(S+P+1)^d.
\]
Polynomial bounds on these quantities and on tube and stage certification
give polynomial total inference cost; the tracker does not discover those
certificates. \hyperref[thm:concrete-deq]{Theorem~\ref*{thm:concrete-deq} (concrete $\tanh$ DEQ)}
instantiates a bounded scalar DEQ with a rational activation evaluator and
finite dyadic updates. Its work envelope is
$O((b+1)(S+b+1)^{12})$, hence $O((S+b+1)^{13})$, with $S$ now denoting
input fraction size. This concrete backend does not supply general
multidimensional gate or training-region certificates.

For training, programmable dormant bilinear channels permit an
output-preserving repair when a loaded Tikhonov pass fails.
\hyperref[thm:tikhonov]{Theorem~\ref*{thm:tikhonov} (Tikhonov dichotomy)},
\hyperref[thm:stable]{Theorem~\ref*{thm:stable} (stable repair)}, and
\hyperref[thm:pass]{Theorem~\ref*{thm:pass} (finite numerical pass)} connect the diagnosis
to a certified interpolation displacement.
\hyperref[thm:main]{Theorem~\ref*{thm:main} (certified training)} then reduces a residual bounded
by $2^\ell$ to $2^{-b}$ using at most $3(b+\ell)$ passes and channels,
with total bit cost at most
\[
 3(b+\ell)A(S+Q+1)^d.
\]
Here $S$ bounds pass query size, $Q$ bounds precision, and $A,d$ are fixed
constants of the selected complete pass backend, including programming,
certificate checking, inference, linear algebra, and rounding.
The \hyperref[ass:class]{certified pass and backend assumptions} require
sound well-posedness and derivative bounds over the full pass regions,
either a loaded invariant or a realizable stable repair, and actual
finite updates meeting the numerical error budgets. When $S,Q,\ell$ and
the complete backend budgets are polynomially bounded in instance length
and $b$, the displayed cost is polynomial. The reserve bound counts
channels; their programmed vectors have additional storage and realization
costs. The inference and training constants $A,d$ belong to their respective
selected backend families and need not coincide.

The \hyperref[tab:deq-comparison]{DEQ numerical comparison} illustrates
the potential computational benefit of a loaded trigger. The
\hyperref[par:feedforward-comparison]{feedforward comparison} also shows
why a channel reserve alone does not ensure successful interpolation.
These heuristic experiments do not validate the certified algorithm's
full contracts. The \hyperref[sec:numerics]{formal verification scope}
describes the quantitative results checked in Lean and the implementation
claims outside that scope.
The \hyperref[sec:matrix-free]{matrix-free implementation discussion}
explains how training can use implicit Jacobian and adjoint products,
iterative regularized solves, and checked residuals without forming the
training Jacobian or an explicit inverse.

The next practical step is to construct efficient gate and region
certificates for larger DEQs. The present results identify exactly where
those certificates enter the complexity guarantees, while the
\hyperref[prop:implicit-stack]{composition bounds} extend the inference
framework to finite stacks with separately controlled budgets.


\begin{thebibliography}{99}
\bibitem{berg2021validated} J.~B.~van den Berg and E.~Queirolo.
A general framework for validated continuation of periodic orbits in
systems of polynomial ODEs. Journal of Computational Dynamics 8(1),
59--97, 2021. \url{https://doi.org/10.3934/jcd.2021004}.
\bibitem{dickson2006} K.~I.~Dickson, C.~T.~Kelley, I.~C.~F.~Ipsen,
and I.~G.~Kevrekidis. Condition Estimates for Pseudo-Arclength Continuation.
arXiv:math/0603716, 2006. \url{https://arxiv.org/abs/math/0603716}.
\bibitem{sittoni2024} P.~Sittoni and F.~Tudisco.
Subhomogeneous Deep Equilibrium Models. ICML, 2024.
\url{https://arxiv.org/abs/2403.00720}.
\bibitem{gabor2024} M.~Gabor, T.~Piotrowski, and R.~L.~G.~Cavalcante.
Positive Concave Deep Equilibrium Models. Proceedings of Machine
Learning Research 235, 14365--14381, 2024.
\url{https://proceedings.mlr.press/v235/gabor24a.html}.
\bibitem{barrett1994templates} R.~Barrett et al.
\emph{Templates for the Solution of Linear Systems: Building Blocks for
Iterative Methods}. SIAM, second edition, 1994.
\url{https://www.netlib.org/templates/templates.html}.
\bibitem{ding2023homoode} S. Ding, T. Cui, J. Wang, and Y. Shi. Two Sides of The Same Coin: Bridging Deep Equilibrium Models and Neural ODEs via Homotopy Continuation. NeurIPS, 2023. \url{https://arxiv.org/abs/2310.09583}.
\bibitem{beltran2013} C. Beltr{\'a}n and A. Leykin. Robust certified numerical homotopy tracking. Foundations of Computational Mathematics 13, 253--295, 2013. \url{https://arxiv.org/abs/1105.5992}.

\bibitem{bai2019} S.~Bai, J.~Z.~Kolter, and V.~Koltun.
Deep equilibrium models. \emph{NeurIPS}, 2019.
\bibitem{winston2020} E.~Winston and J.~Z.~Kolter.
Monotone operator equilibrium networks. \emph{NeurIPS}, 2020.
\bibitem{hu2022} E.~J.~Hu et al.
LoRA: Low-rank adaptation of large language models. \emph{ICLR}, 2022.
\bibitem{blum1992} A.~L.~Blum and R.~L.~Rivest.
Training a 3-node neural network is NP-complete.
\emph{Neural Networks} 5(1):117--127, 1992.
\bibitem{tikhonov1963} A.~N.~Tikhonov.
Solution of incorrectly formulated problems and the regularization method.
\emph{Soviet Mathematics Doklady} 4:1035--1038, 1963.
\bibitem{levenberg1944} K.~Levenberg.
A method for the solution of certain non-linear problems in least squares.
\emph{Quarterly of Applied Mathematics} 2:164--168, 1944.
\bibitem{marquardt1963} D.~W.~Marquardt.
An algorithm for least-squares estimation of nonlinear parameters.
\emph{SIAM Journal} 11(2):431--441, 1963.
\bibitem{hager1989} W.~W.~Hager.
Updating the inverse of a matrix. \emph{SIAM Review} 31(2):221--239, 1989.

\bibitem{bai2020} S.~Bai, V.~Koltun, and J.~Z.~Kolter.
Multiscale deep equilibrium models. \emph{NeurIPS}, 2020.
\url{https://arxiv.org/abs/2006.08656}.
\bibitem{elghaoui2021} L.~El Ghaoui, F.~Gu, B.~Travacca, A.~Askari, and A.~Tsai.
Implicit deep learning. \emph{SIAM Journal on Mathematics of Data Science}
3(3):930--958, 2021. \url{https://doi.org/10.1137/20M1358517}.
\bibitem{jafarpour2021} S.~Jafarpour, A.~Davydov, A.~V.~Proskurnikov, and F.~Bullo.
Robust implicit networks via non-Euclidean contractions. \emph{NeurIPS}, 2021.
\url{https://arxiv.org/abs/2106.03194}.
\bibitem{ryu2016} E.~K.~Ryu and S.~Boyd.
A primer on monotone operator methods. \emph{Applied and Computational Mathematics}
15(1):3--43, 2016. \url{https://web.stanford.edu/~boyd/papers/monotone_primer.html}.
\bibitem{bai2021} S.~Bai, V.~Koltun, and J.~Z.~Kolter.
Stabilizing equilibrium models by Jacobian regularization. \emph{ICML},
PMLR 139:554--565, 2021. \url{https://proceedings.mlr.press/v139/bai21b.html}.
\bibitem{fung2022} S.~Wu Fung, H.~Heaton, Q.~Li, D.~McKenzie, S.~Osher, and W.~Yin.
JFB: Jacobian-Free Backpropagation for implicit networks. \emph{AAAI}, 2022.
\url{https://arxiv.org/abs/2103.12803}.
\bibitem{ramzi2022} Z.~Ramzi, F.~Mannel, S.~Bai, J.-L.~Starck, P.~Ciuciu, and T.~Moreau.
SHINE: SHaring the INverse Estimate from the forward pass for bi-level
optimization and implicit models. \emph{ICLR}, 2022.
\url{https://arxiv.org/abs/2106.00553}.
\bibitem{kawaguchi2021} K.~Kawaguchi.
On the theory of implicit deep learning: Global convergence with implicit
layers. \emph{ICLR}, 2021. \url{https://arxiv.org/abs/2102.07346}.
\bibitem{gao2022} T.~Gao, H.~Liu, J.~Liu, H.~Rajan, and H.~Gao.
A global convergence theory for deep ReLU implicit networks via
over-parameterization. \emph{ICLR}, 2022.
\url{https://arxiv.org/abs/2110.05645}.
\bibitem{ling2023} Z.~Ling, X.~Xie, Q.~Wang, Z.~Zhang, and Z.~Lin.
Global convergence of over-parameterized deep equilibrium models.
\emph{AISTATS}, PMLR 206:767--787, 2023.
\url{https://proceedings.mlr.press/v206/ling23a.html}.
\bibitem{truong2025} L.~V.~Truong.
Global convergence rate of deep equilibrium models with general activations.
\emph{Transactions on Machine Learning Research}, 2025.
\url{https://openreview.net/forum?id=XPREcQlAM0}.
\bibitem{chen2016} T.~Chen, I.~Goodfellow, and J.~Shlens.
Net2Net: Accelerating learning via knowledge transfer. \emph{ICLR}, 2016.
\url{https://arxiv.org/abs/1511.05641}.
\bibitem{wei2016} T.~Wei, C.~Wang, Y.~Rui, and C.~W.~Chen.
Network morphism. \emph{ICML}, PMLR 48:564--572, 2016.
\url{https://proceedings.mlr.press/v48/wei16.html}.
\bibitem{lawton2024} N.~Lawton, A.~Galstyan, and G.~Ver Steeg.
Learning morphisms with Gauss--Newton approximation for growing networks.
\emph{OPT2024: Workshop on Optimization for Machine Learning}, 2024.
\url{https://opt-ml.org/papers/2024/paper75.pdf}.
\bibitem{sato2026} N.~Sato and H.~Iiduka.
Lipschitz multiscale deep equilibrium models: A theoretically guaranteed
and accelerated approach. \emph{arXiv preprint}, 2026.
\url{https://arxiv.org/abs/2602.03297}.
\bibitem{silva2026} J.~L.~Lima de Jesus Silva.
Response renormalization for critical deep equilibrium models.
\emph{arXiv preprint}, 2026. \url{https://arxiv.org/abs/2608.23725}.
\bibitem{mastrogiuseppe2018} F.~Mastrogiuseppe and S.~Ostojic.
Linking connectivity, dynamics, and computations in low-rank recurrent
neural networks. \emph{Neuron} 99(3):609--623.e29, 2018.
\url{https://doi.org/10.1016/j.neuron.2018.07.003}.
\bibitem{liang2018} S.~Liang, R.~Sun, J.~D.~Lee, and R.~Srikant.
Adding one neuron can eliminate all bad local minima. \emph{NeurIPS}, 2018.
\url{https://arxiv.org/abs/1805.08671}.
\end{thebibliography}
\end{document}